\documentclass[]{article}
\usepackage{proceed2e}

\usepackage[usenames,dvipsnames]{xcolor}
\usepackage{hyperref}
\usepackage{url}

\usepackage{amsmath, amsthm, amssymb,algorithmic,algorithm,array}
\usepackage[pdftex]{graphicx}
\usepackage{epstopdf}
\usepackage{amsfonts}
\usepackage{tikz}
\usepackage{natbib}
\usepackage{multibib}
\usepackage{tabularx}

\newcites{main}{References}
\newcites{appendix}{Supplementary References}

\usepackage{caption}
\usepackage{subcaption}

\allowdisplaybreaks
\newtheorem{theorem}{\textbf{Theorem}}[section]

\newcommand{\dt}[2]{\left\langle #1,#2 \right\rangle}

\DeclareMathOperator{\Ne}{Ne}

\DeclareMathOperator{\vvec}{vec}
\DeclareMathOperator{\tr}{tr}

\graphicspath{{../../fig/}}

\usepackage{times}
\title{Bethe Learning of Conditional Random Fields via MAP Decoding}

\author{} 

\author{
  {\bf Kui Tang} \\
  Columbia University \\
  \And
  {\bf Nicholas Ruozzi} \\
  UT Dallas \\
  \And
  {\bf David Belanger}\\
  UMass Amherst \\
  \And
  {\bf Tony Jebara} \\
  Columbia University
}

\begin{document}

\maketitle

\begin{abstract}
Many machine learning tasks can be formulated in terms of predicting structured outputs. In frameworks such as the structured support vector machine (SVM-Struct) and the structured perceptron, discriminative functions are learned by iteratively applying efficient maximum a posteriori (MAP) decoding. However, maximum likelihood estimation (MLE) of probabilistic models over these same structured spaces requires computing partition functions, which is generally intractable. This paper presents a method for learning discrete exponential family models using the Bethe approximation to the MLE. Remarkably, this problem also reduces to iterative (MAP) decoding. This connection emerges by combining the Bethe approximation with a Frank-Wolfe (FW) algorithm on a convex dual objective which circumvents the intractable partition function. The result is a new single loop algorithm MLE-Struct, which is substantially more efficient than previous double-loop methods for approximate maximum likelihood estimation. Our algorithm outperforms existing methods in experiments involving image segmentation, matching problems from vision, and a new dataset of university roommate assignments.
\end{abstract}
			
\section{INTRODUCTION}
Learning the parameters of a Markov random field (MRF) or a conditional random field (CRF) is a ubiquitous problem in machine learning and related fields.  Often, the parameters are learned via regularized maximum likelihood estimation (MLE) and then prediction is performed via maximum a-posteriori (MAP) or marginal inference\footnote{In this paper, MAP inference refers to predicting an output $Y$ given parameters $\theta$, while MLE learning refers to estimating $\theta$ given observations $(Y^{1},X^{1}), \ldots,(Y^{m},X^{m})$, optionally including quadratic regularization.} As the log-likelihood is concave, it can in principle be maximized by gradient ascent.  However, this requires repeatedly computing gradients of the log-partition function, which in general is intractable. One can circumvent this difficulty by using surrogates for the log-partition function~\citepmain{sutton2005, ganapathi2008, domke2013} or by approximating the partition function using sampling~\citepmain{petterson2009exponential,papandreou2011perturb}. 

Alternatively, one can avoid likelihoods entirely, and use methods such as the structured perceptron or structured support vector machines (SVM-Struct) that rely only on a MAP solver~\citepmain{collins2002discriminative,roller2004max,tsochantaridis2004support, finley2008training}.  Such methods can often be quite accurate and are typically faster than approximate MLE, since MAP, or relaxations thereof, can be performed quickly using sophisticated combinatorial solvers. By using such solvers as black boxes, MAP-based training methods also offer users an attractive abstraction between the learning problem and the optimization algorithm.   On the other hand, MLE remains a primary goal for many practitioners, since it may yield superior predictive accuracy, offers parameter values with increased interpretability and statistical properties, and supports test-time marginal inference.

In this work, we introduce MLE-Struct, a novel approximate MLE algorithm that also only requires access to a MAP solver. We combine Bethe-style convex free energies with the Frank-Wolfe (FW) method~\citepmain{frank1956algorithm,jaggi2013revisiting}.   A naive application of FW for approximate MLE would perform approximate marginal inference using repeated calls to MAP, as in the experiments of~\citetmain{sontag2007new}, and then use these marginals to perform a single gradient step on the parameters. This double-loop algorithm requires a significant number of MAP solver calls, especially if very accurate answers are required.  Our approach achieves fast learning by avoiding this costly double loop structure.  First, we employ a generic reweighted entropy approximation technique that yields convex Bethe-style surrogate likelihoods for any underlying undirected graphical model. Then, we construct a constrained, convex dual problem for this approximate maximum likelihood objective.  We demonstrate that the approximate dual problem can be minimized efficiently using FW:  each of the linear subproblems that are solved as part of the algorithm can be formulated as separate approximate or exact MAP inference tasks on each training example.  Finally, we introduce a technique to accelerate the line search subroutine of FW by precomputing certain data-dependent terms. 

We can also use FW to perform test-time marginal inference using the procedure of~\citetmain{sontag2007new}.  Therefore, at both train and test time we can interact with our underlying problem structure using only a MAP routine.  This allows us to design fast approximate learning and prediction algorithms for a wide variety of settings in which efficient approximate/exact MAP solvers exist:  bipartite/general matching and b-matching problems (via the max-flow and blossom algorithms), pairwise binary graphical models (via QPBO), planar Ising models with no external field (via a reduction to matching)~\citepmain{schraudolph08efficient}, among others.

We apply our method to learn pairwise binary CRFs and distributions over matchings on both bipartite and general graphs. Our method provides good predictive performance while often solving the approximate MLE problem significantly faster with fewer numerical instabilities than other approximate MLE methods.  We also apply our method to a new dataset of housing preferences and roommate assignments of university students to predict good freshmen roommate assignments.

\section{BACKGROUND AND RELATED WORK}
\label{sec:background}

We consider conditional random fields where, in addition to samples $Y^{(1)},\ldots,Y^{(M)}$ from some discrete space $\mathcal{Y}$, we also observe feature vectors $X^{(1)},\ldots,X^{(M)}\in\mathcal{X}$ \citepmain{lafferty2001conditional}.  In this case, the conditional probability of the $m^{\text{th}}$ sample has the form
\begin{align}
p(Y^{(m)}|X^{(m)};\theta) = \frac{\exp(\langle\phi(X^{(m)},Y^{(m)}),\theta\rangle)}{Z^{(m)}(X;\theta)}\label{eq:crf}
\end{align}
where $\phi$ is a vector of sufficient statistics, $\theta$ is a vector of parameters, and the partition function is given by 
\[Z(X^{(m)};\theta) = \sum_{Y\in\mathcal{Y}} \exp(\langle\phi(X^{(m)},Y),\theta\rangle).\]
In typical applications, the joint probability distribution factors over a hypergraph $G = (V,\mathcal{A})$ where $\mathcal{A}$ is a collection of subsets of $V$.  For ease of presentation, we assume that
\[\phi(X,Y) = \{\phi_i(X,Y_i)|_{i\in V}, \phi_{\alpha}(X,Y_\alpha)|_{\alpha\in \mathcal{A}}\}.\]
That is, for $Y\in\mathcal{Y}$ and $X\in\mathcal{X}$,
\begin{align*}
p(Y|&X;\theta) =\\
&\frac{\exp\Big(\sum_{i\in V}\theta_i\phi_i(X,Y_i) + \sum_{\alpha\in\mathcal{A}} \theta_\alpha \phi_\alpha(X,Y_\alpha)\Big)}{Z(X;\theta)}.
\end{align*}
These models include both Potts and Ising models, as well as log-linear distributions over matchings~\eqref{eq:match}.

Given $M$ observations $\{Y^{(1)},\ldots,Y^{(M)}\}$ with corresponding feature vectors $\{X^{(1)},\ldots,X^{(M)}\}$ where $p(Y^{(m)}|X^{(m)};\theta)$ is of the form \eqref{eq:crf}, we would like to learn $\theta$ by maximizing the log-likelihood of the observations plus a quadratic regularizer.
\begin{align*}
\ell(\theta;X^{(1:M)},Y^{(1:M)}) &= \sum_{m=1}^M \ell(\theta;X^{(m)},Y^{(m)}) - \frac{\lambda}{2}\|\theta\|^2
\end{align*}
Here,
\begin{align*}
\ell(\theta;X^{(m)},Y^{(m)})  = \langle\phi(X^{(m)},Y^{(m)}),\theta\rangle - \log Z(X^{(m)}).
\end{align*}

\subsection{CONVEX FREE ENERGY APPROXIMATIONS}
\label{sec:energyapprox}

The central challenge in maximum likelihood estimation is computing the partition function $Z(X^{(m)}, \theta)$ for each sample $m$ at each iteration. In this work, we approximate the partition function in order to make the learning problem tractable. We begin with the Bethe free energy, a standard approximation to the so-called Gibbs free energy that is motivated by ideas from statistical physics.  The approximation has been generalized to include different \textit{counting numbers} that result in alternative entropy approximations \citepmain{weissdbcnt}.  We focus on a restricted set of counting numbers that result in a family of \emph{convex} reweighted free energies.

The reweighted free energy at temperature $T = 1$ is specified by a polytope approximation $\mathcal{T}$, the hypergraph $G = (V,\mathcal{A})$, an entropy approximation $H_\rho$, and a vector of counting numbers $\rho$ (henceforth referred to as reweighting parameters).
\begin{align}
\log F_\rho(\tau, X; \theta) \triangleq E(\tau, X;\theta) - H_\rho(\tau),\label{eq:rfe}
\end{align}
where the energy is given by
\begin{align*}
E(\tau, X;\theta) \triangleq & -\sum_{i\in V} \sum_{Y_i} \tau_i(Y_i)\theta_i\phi_i(X,Y_i)\\
&\:\: - \sum_{\alpha\in\mathcal{A}} \sum_{Y_\alpha} \tau_\alpha(Y_\alpha)\theta_\alpha\phi_\alpha(X,Y_\alpha), \
\end{align*}
the entropy approximation is given by
\begin{align*}
H_\rho(\tau) \triangleq &-\sum_{i\in V} \sum_{y_i} \tau_i(y_i)\log \tau_i(y_i)\\
&\:\: -\sum_{\alpha\in\mathcal{A}} \sum_{y_\alpha} \rho_\alpha \tau_\alpha(y_\alpha)\log \frac{\tau_\alpha(y_\alpha)}{\prod_{i\in\alpha} \tau_i(y_i)},
\end{align*}
and $\tau$ is restricted to lie in an outer bound of the marginal polytope known as the local (marginal) polytope,
\begin{align*}
\mathcal{T} \triangleq \{ &\tau\geq 0 : \mbox{for all } i\in V, \sum_{Y_i} \tau_{i}(Y_i) = 1\\
&\text{for all } \alpha\in\mathcal{A}, i\in \alpha, Y_i, \sum_{Y_{\alpha\setminus \{i\}}} \tau_\alpha(Y_\alpha) = \tau_i(Y_i)  \}.
\end{align*}
The reweighted partition function is then computed by minimizing \eqref{eq:rfe} over $\mathcal{T}$
\[Z_\rho(X;\theta) \triangleq \exp(-\min_{\tau\in\mathcal{T}} F_\rho(\tau, X; \theta)).\]
Setting $\rho_\alpha =1$ for each $\alpha\in\mathcal{A}$ recovers the typical Bethe free energy approximation. The reweighting parameters can always be chosen so that the approximate free energy is convex \citepmain{margwain, hazan2012convergent, ruozzi2013}.  For example, the tree-reweighted belief propagation algorithm (TRW) chooses the reweighting parameters so that they correspond to (hyper)edge appearance probabilities of a collection of spanning (hyper)trees.

\subsection{SADDLE POINT FORMULATION}
We approximate the exact partition function in the MLE objective with a reweighted free energy approximation of the form~\eqref{eq:rfe}.  This results in the saddle-point problem
\begin{align}
\max_{\theta} \min_{\tau^{(1:M)}\in\mathcal{T}}\Bigg[&\sum_{m=1}^M \Big[\langle\phi(X^{(m)},Y^{(m)}),\theta\rangle\hspace{2.8cm}\nonumber\\
&\: - \log F_\rho(\tau^{(m)}, X^{(m)};\theta)\Big] - \frac{\lambda}{2} \|{\theta}\|^2\Bigg].\hspace{-.5cm} \label{eq:saddle} 
\end{align}
\citetmain{heinglob2011} investigated unregularized likelihoods of this form for MRFs, and demonstrated that convexity of the Bethe free energy guarantees that the empirical marginals satisfy a moment matching condition:  the empirical marginals minimize the Bethe free energy at the $\theta$ that maximizes the approximate log-likelihood.  Moment matching is not necessarily achieved for general MRFs when the reweighted approximation is not convex \citepmain{heinglob2011}.  \citetmain{margwain} investigated the use of TRW for learning in pairwise binary graphical models.  They observed that the parameters learned via TRW were more robust to the addition of new data than those learned by BP.  This robustness of convex free energy approximations for learning can be made theoretically precise \citepmain{wainwright2006wrong}.

If we compute the partition function via an iterative procedure, then solving~\eqref{eq:saddle} necessarily requires a double-loop algorithm, which can be expensive for large datasets.   The existing work on {\it Bethe learning} has sought to design more efficient approximate learning algorithms.  \citetmain{sutton2005} proposed a piecewise training scheme whereby the graph is divided into smaller subgraphs over which the partition function can be efficiently computed exactly or approximately.  These results are then combined to approximate the true partition function.  Because the subproblems are typically much smaller, the procedure is quite fast but can be inaccurate if the pieces are too small \citepmain{ganapathi2008}. For bipartite matchings, one can obtain an unbiased but noisy gradient of the log-likelihood by utilizing an $O(|V|^4 \log V)$ perfect sampler algorithm due to \citetmain{huber08fast}. \citetmain{petterson2009exponential} use this approach for ranking and graph matching problems, but limited themselves to 20 vertices, and each observation required its own set of samples. \citetmain{domke2013} proposed performing MLE using a small, fixed number of TRW iterations as part of a procedure to estimate the gradient.  However, if TRW is not converging quickly (i.e., a reasonable solution is not obtained after running for a fixed number of iterations), the resulting procedure can fail to converge. \citetmain{vishwanathan2006accelerated} proposed improving the convergence in the outer loop using accelerated gradient methods. All of the above methods rely on a double loop.

\section{APPROXIMATE MLE}

We now consider a  convex \textit{dual} reformulation of~\eqref{eq:saddle} that applies to convex free energies and yields a new, fast learning algorithm.  We first note that~\eqref{eq:saddle} is concave in the variables being maximized and convex in the variables being minimized, and one set of variables (the $\tau$) are constrained to a compact domain. We can thus invoke Sion's minimax theorem \citepmain{sion1958} to reverse the $\max$ and $\min$ operators. Next, we can analytically solve for the optimal $\theta$ in terms of fixed $\tau^{(1)},\ldots,\tau^{(M)}\in\mathcal{T}$. Setting the gradient with respect to $\theta$ equal to zero in \eqref{eq:saddle} yields
\begin{align}
\theta^*_i(\tau^{(1:M)}) = \frac{1}{\lambda}\Bigg( & \sum_m \Big[\phi_i(X^{(m)}, Y^{(m)}_i)  -  \nonumber\\
  &  \sum_{Y_i} \tau^{(m)}_i(Y_i) \phi_i(X^{(m)}, Y_i)\Big]\Bigg)  \label{eq:tau2theta_a} \\
\theta^*_\alpha(\tau^{(1:M)}) = \frac{1}{\lambda}\Bigg( & \sum_m \Big[\phi_\alpha(X^{(m)}, Y^{(m)}_\alpha) - \nonumber\\
    &\sum_{Y_\alpha} \tau^{(m)}_\alpha(Y_\alpha) \phi_\alpha(X^{(m)}, Y_\alpha)\Big]\Bigg). \label{eq:tau2theta_b} 
\end{align}
Finally, substituting these back into \eqref{eq:saddle} yields the following optimization problem over the local marginal polytope.
\begin{align}
\min_{\tau^{(1:M)}\in\mathcal{T}} \frac{1}{2\lambda}\|\theta^*(\tau^{(1:M)})\|^2 - &\sum_m H_\rho(\tau^{(m)})\nonumber\\
&\triangleq \min_{\tau^{(1:M)}\in\mathcal{T}} L(\tau^{(1:m)})
\label{eq:single}
\end{align}
The linearly-constrained convex objective~\eqref{eq:single} can be minimized via general convex optimization techniques, such as the ellipsoid method---though this can be slow in practice. In the sequel, we minimize this objective with the Frank-Wolfe algorithm (FW). 

Convex free energies can be obtained using the reweighting techniques above, and whenever the graph is a tree the standard Bethe free energy is both convex and exact. In principle, a similar argument can be made for any convex approximation of the partition function, though we only focus on convex Bethe-style approximations in this work.

The MLE objective is dual to the maximum entropy problem, and recent work on approximate MLE has focused on different families of entropy approximations \citepmain{wainwright2008graphical}. \citetmain{ganapathi2008} also followed a maximum entropy approach to the approximate MLE problem.  They proposed approximating the entropy objective using the Bethe entropy approximation (i.e., $H_1$), but specifically avoided convex entropy approximations.  Unfortunately, this results in a non-convex optimization problem in general, for which the authors use the concave-convex procedure. Other recent work approximated the MLE  problem using convex free energies, but did not consider the maximum entropy approach \citepmain{domke2013}.

\subsection{FRANK-WOLFE ALGORITHM FOR MAXIMUM LIKELIHOOD LEARNING}
\label{sec:fw-match}
  
 Following \citetmain{jaggi2013revisiting}, FW minimizes 
a general convex function $f(x)$ over a convex set $\mathcal{X}$ via a sequence of iterates defined by
\begin{align}
{s}_{t} &= \arg \min_{x \in \mathcal{X}} \langle x,  \nabla
f(x_{t-1})\rangle \label{fw:subproblem} \\
x_{t} &= (1 - \gamma_t ) x_{t-1} + \gamma_t {s}_{t}, \label{fw-update}
\end{align}
where the step-size, $\gamma_t$, is either selected using line search or is fixed at
$\frac{2}{2+t}$.

For the objective function in \eqref{eq:single}, each step requires minimizing a linear objective over a linear set of constraints.
\begin{align}
s_t = {\arg\min}_{\tau^{(1)},\ldots, \tau^{(M)}\in\mathcal{T}} \langle\tau^{(1:m)}, \nabla L(\tau^{(1:m)}_{t-1})\rangle\label{fw:iterate}
\end{align}
Since the constraints are separable across training examples, \eqref{fw:iterate} decouples into $M$ independent linear programs (LPs) that can be solved in parallel.  Depending on the specific application, purely combinatorial methods or reweighted message-passing algorithms may provide faster and more space efficient alternatives to generic LP solvers.  Note that~\eqref{eq:single} could not be solved with projected gradient algorithms efficiently, since projection onto the local or marginal polytopes is not tractable. 

Despite the ability to perform~\eqref{fw:iterate} in parallel, it can still be prohibitive for large sample sizes. For convex optimization problems over separable constraint spaces, \citetmain{lacoste2013block} propose a block-coordinate FW algorithm (BCFW).  The BCFW procedure performs the FW iteration over a randomly selected $m\in\{1,\ldots,M\}$ and leaves the remaining coordinates untouched.  BCFW requires less work at each iteration, but the asymptotic rate of convergence remains the same as that for the standard FW algorithm \citepmain{lacoste2013block}.  This block coordinate approach is known to outperform FW for the SVM-Struct problem.  Technical details concerning the convergence of FW for this problem, including methods to bound the convergence rate, can be found in Appendix \ref{app:fw}. Both the FW and BCFW versions of our algorithm, MLE-Struct, are described in Algorithm~\ref{alg:fw-learning}.Line search can be accelerated by precomputing quadratic terms as discussed in~\ref{app:line-search}.

\begin{algorithm}[t]
  \caption{MLE-Struct:  Frank-Wolfe Approximate Maximum-Likelihood Learning}   
\begin{algorithmic}
\STATE {\bfseries Input:} training examples $\{(X^{(m)},Y^{(m)}\}$, reweighting parameters $\rho \in [0,1]^n$,  regularizer $\lambda$\\ 
   \STATE {\bfseries Output:} Approximate maximum likelihood $\theta$. 
 \STATE {\bfseries Initialization:} Set each $\tau^{(m)}$ uniformly. 
   \REPEAT 
   \FOR{$\forall m$ in parallel (batch) or $m$ chosen uniformly at random (block)}
   \STATE $s^{(m)}_t = {\arg\min}_{\tau^{(M)}\in\mathcal{T}} \langle\tau^{(m)}, \nabla^{(m)} L(\tau^{(1:m)}_{t-1})\rangle$
   \STATE  Set $\gamma = \frac{2}{2 + t}$ for batch and $\gamma = \frac{2M}{2M + t}$ for block or use line search. 
   \STATE $\tau_t = (1 - \gamma)\tau_{t-1}^{(m)} + \gamma s_t^{(m)}$
   \ENDFOR
   \UNTIL{converged}
\STATE  Set $\theta$ using~\eqref{eq:tau2theta_a} and~\eqref{eq:tau2theta_b}. 

\end{algorithmic}
\label{alg:fw-learning}

\end{algorithm}

\subsection{FRANK-WOLFE FOR MARGINAL INFERENCE}
\label{sec:fw-marg}

In many applications, computing marginals is useful at test time, as well as during learning. Fortunately, we can use FW to perform marginal inference, thus maintaining our ability to interact with the underlying model only through a MAP solver. Specifically, approximate reweighted Bethe marginals can be obtained by minimizing~\eqref{eq:rfe} with respect to $\tau$, which is a convex problem suitable for FW. The technique was first used in~\citetmain{sontag2007new}, using a generic LP solver for MAP. 

In Appendix~\ref{sec:fw-inf}, we provide experiments on CRFs defined over bipartite matchings (see Section~\ref{sec:perms}), demonstrating the favorable accuracy and speed of FW-based inference versus the BP algorithm of~\citetmain{huang2009approximating} and an instance of the Perturb-and-MAP framework designed specifically for matchings~\citepmain{KeLi2013}. We find that FW outperforms Perturb-and-MAP in terms of both accuracy and convergence speed. FW and BP minimize the same objective, since the Bethe entropy is convex for matchings, so we compare them purely in terms of speed. We find that FW is preferable to BP in most regimes, except when extremely precise optimization is required. 

\section{APPLICATIONS AND EXPERIMENTS}
\label{sec:expts}

We apply the MLE-Struct framework to a variety of exponential family models defined over different combinatorial structures, including grid CRFs for image segmentation, bipartite matchings in vision applications, and general perfect matchings for a university roommate assignment problem. For CRFs, MAP inference is intractable, but we can efficiently solve the LP relaxation, which is equivalent to MAP inference over the local polytope with QPBO~\citepmain{rother2007qpbo}. This means our estimated pseudomarginals will not be globally consistent, but the procedure can still yield accurate predictions~\citepmain{margwain}. For the matching problem, we use efficient max-flow solvers to obtain exact MAP solutions (i.e., over the marginal polytope)~\citepmain{goldberg1995efficient,kolmogorov2009blossom}. In this case, our estimated pseudomarginals will be globally consistent. Appendix~\ref{app:expt-details} details the data sources, feature extraction, and machine setup.

\subsection{PERMANENTS AND PERFECT MATCHINGS}
\label{sec:perms}
We first consider the problem of learning distributions over perfect matchings of a given graph.  For a graph $G = (V,E)$ and edge weights $w_{ij} \in \mathbb{R}$, the probability of observing a particular matching is
\begin{align}
f(Y;w) = \frac{1}{Z(W)} \exp\Big(\frac{1}{2}\tr(WY)\Big)\label{eq:match}
\end{align}
where $Y$ is the adjacency matrix of a perfect matching in $G$, $W$ is a weighted adjacency matrix of $G$, and $Z(W)$ is the partition function. Each entry of $W$ is a function of edge-wise features. Our formulation can be relaxed to distributions over all matchings by allowing $Y$ to correspond to the adjacency matrix of any (not necessarily perfect) matching.

When $G$ is bipartite, the partition function is the permanent of the matrix of edge weights and is thus \#P-hard to compute~\citepmain{valiant1979complexity}. Although the partition function can be computed to any given accuracy using a fully polynomial randomized approximation scheme \citepmain{jerrum2004polynomial}, such algorithms are impractical for graphs of any significant size.

In practice, $W$ is unknown and must be learned from data. We can learn a generative model by estimating $W$ directly, or a conditional model by first assuming that $W$ is the linear combination of some feature maps and then learning the weights. For concreteness, suppose we have $K$ features, and for the $k^\text{th}$ feature we have a $|V|\times|V|$ matrix $F_k$. Let $\theta \in \mathbb{R}^K$ be our model parameters, so that the weight on edge $(i,j)\in E$ is $W_{ij} = \sum_{k=1}^K \theta_K F^k_{ij}$. Then the conditional likelihood is 
\begin{align}
p(Y;F^{1:K},\theta) = \frac{\exp{\Big(\frac{1}{2}\sum_{k = 1}^K \theta_k \tr(F^k Y)\Big)}}{Z(F^{1:K},\theta)}\label{eq:gm}
\end{align}
\begin{theorem}
\label{thm:conv}
For any $\rho \in [0,1]^{|V|}$, any graph (bipartite or general), and any matching (perfect or imperfect), the reweighted free energy~\eqref{eq:rfe} is convex over the local polytope.
\end{theorem}
Theorem~\ref{thm:conv} is proven in Appendix~\ref{app:conv}. By inclusion, it implies~\eqref{eq:rfe} is also convex over the marginal polytope. This generalizes earlier known results of convexity for bipartite perfect matchings \citepmain{betheperm, chertkov2013}. Due to the convexity of the Bethe entropy and the availability of high-quality maximum-weight matching solvers, Algorithm~\ref{alg:fw-learning} is well-suited to the approximate MLE task.  A derivation of the specific form of~\eqref{eq:single} for matchings and a technique for making the  associated line search particularly efficient by precomputing certain data-dependent terms can be found in Appendix~\ref{app:fwmatch}.

\subsubsection{SYNTHETIC BIPARTITE MATCHINGS}

\begin{figure}[t!]
  \centering
  \begin{subfigure}{\columnwidth}
    \includegraphics[width=\columnwidth]{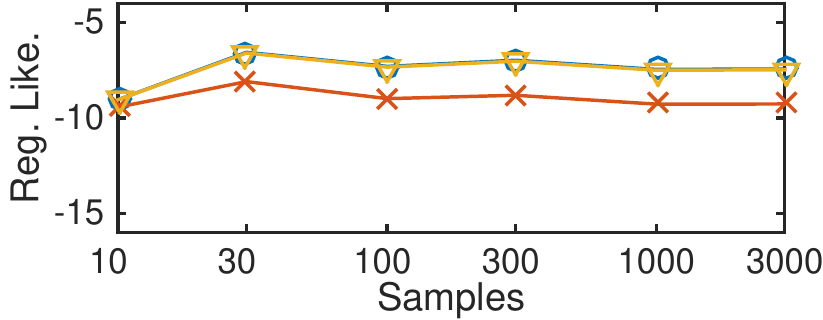}
    \caption{\emph{High SNR problem}: Average (divided by sample size) \emph{true} regularized log-likelihood evaluated at the exact MLE as well as at the parameters that maximize the RW and Bethe approximate log-likelihoods. Higher is better, with the \emph{Exact} curve being the upper bound. The data are nearly as probable under the Bethe estimator as they are under the exact MLE.}
  \end{subfigure}

  \begin{subfigure}{\columnwidth}
    \includegraphics[width=\columnwidth]{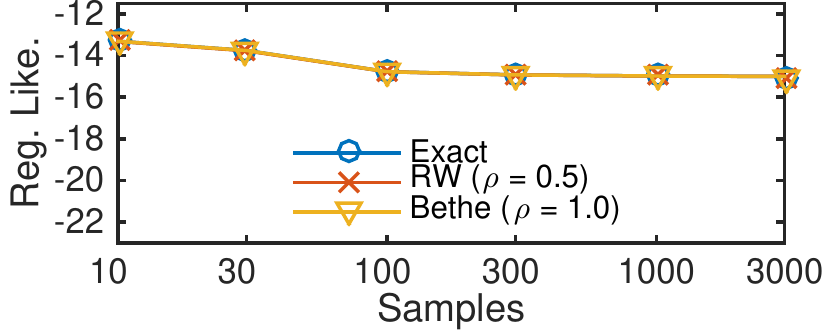}
    \caption{\emph{Low SNR problem:} Same plot as (a). All methods perform comparably in this setting.} 
  \end{subfigure}

  \begin{subfigure}{\columnwidth}
    \includegraphics[width=\columnwidth]{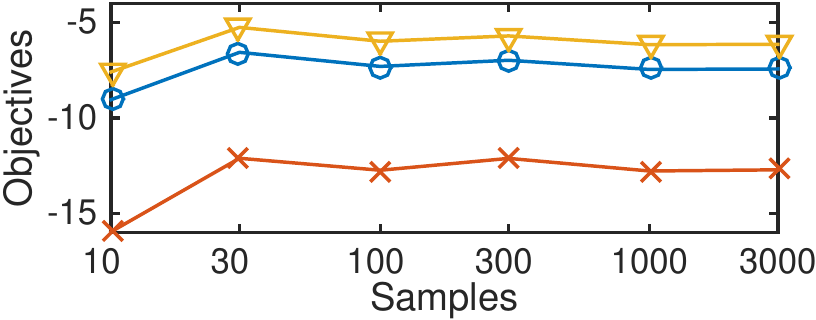}
    \caption{\emph{High SNR problem:} Optimal values of regularized true, and approximate log-likelihoods under the RW and Bethe approximations. The true likelihood is always bounded by the approximations.}
  \end{subfigure}

  \begin{subfigure}{\columnwidth}
    \includegraphics[width=\columnwidth]{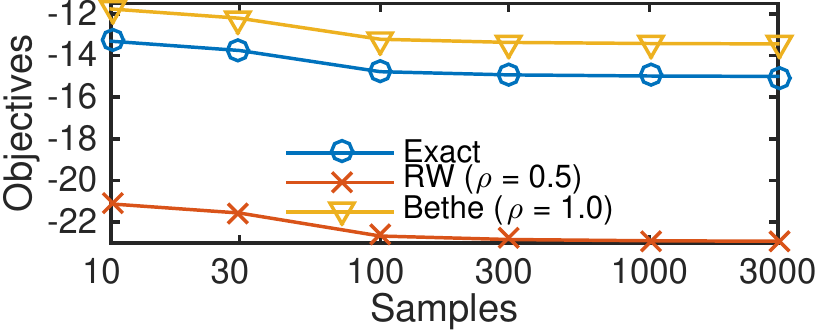}
    \caption{\emph{Low SNR problem:} Same as (c)}
  \end{subfigure}

  \caption{Exact likelihoods of the Bethe and RW estimates, and sandwich bounds on the likelihood.}
  \label{fig:synmat}
\end{figure}

We begin with a synthetic experiment using the flexibility of MLE-Struct to analyze the accuracy of various entropy approximations for matchings. We sample $10 \times 10$ bipartite matchings from the distribution~\eqref{eq:gm}. We explore two choices of the weight matrix $W$: one in the \emph{high SNR regime} with $-2$ on the off-diagonals and $0$ on the diagonals, and one in the \emph{low SNR regime} $-0.5$ off-diagonal and $0$ on-diagonal. 

Our problems are small enough that we can compute exact partition functions and their gradients with Ryser's algorithm~\citepmain{ryser1963}.  Hence, we can perform exact MLE with gradient descent. We can also evaluate the true (regularized) likelihood of our estimates. We ran Algorithm~\ref{alg:fw-learning} with $\rho = 1$ and called this result the \emph{Bethe estimator.} In addition, setting $\rho = 0.5$ for this problem guarantees a concave entropy approximation and an upper bound on the partition function \citepmain{weissconv}.  We also ran Algorithm~\ref{alg:fw-learning} at this setting and denote the result as the \emph{RW estimator}.

Figure~\ref{fig:synmat}(a) displays the average regularized log-likelihood of each estimator, higher being better and the \emph{Exact} curve being an upper bound. In both low and high SNR regimes, the Bethe estimator is superior to the RW estimator. Reweighted entropies such as the one chosen here are known to perform poorly as estimators of the true partition function as compared to belief propagation.  Interestingly, although the objective values of the estimates are different in each case, in the low SNR regime, all estimation methods produce about the same likelihood.

Our framework can also be used to bound the value of the true likelihood. First, since $Z_\text{RW}$ provides an upper bound on $Z$~\citepmain{margwain}, we have $Z_{\text{RW}}(W) = Z_{0.5}(W) \geq Z(W)$ for any $W$. Second, for matchings, we have $Z_B(W) = Z_{1}(W) \leq Z(W)$ \citepmain{gurvitsnew}. Therefore, we have

\begin{equation}
  \log \ell_{\text{RW}}(W) \leq \log \ell(W) \leq \log \ell_B(W)
  \label{eq:sandwich-all}
\end{equation}
for all $W$. Moreover, $\ell_{\text{RW}}$ and $\ell_B$ are \emph{global} bounds on the maximum likelihood, so the inequalities also hold at their respective optima. That is,
\begin{equation}
  \log \ell_{\text{RW}}(W^*_{\text{RW}}) \leq \log \ell(W^*) \leq \log \ell_B(W^*_{B})
  \label{eq:sandwich-opt}
\end{equation}
where $W^*_{\text{RW}}$ is the RW estimator, $W^*$ is the regularized MLE, and $W^*_{B}$ is the Bethe estimator. We plot the quantities of~\eqref{eq:sandwich-opt} in Figure~\ref{fig:synmat}(b). We can also use~\eqref{eq:sandwich-all} to obtain upper and lower bounds of $\ell(W^*_B)$ and $\ell(W^*_\text{RW})$ by using FW for inference to compute $\ell_B(W^*_{\text{RW}})$ and $\ell_{\text{RW}}(W^*_{B})$, since $\ell_B(W^*_B)$ and $\ell_{\text{RW}}(W^*_{\text{RW}})$ will already be available upon convergence of Algorithm~\ref{alg:fw-learning}. Appendix~\ref{app:expt-details} shows these results.

\subsubsection{GRAPH MATCHINGS}
\begin{figure}[t!]
\centering
\scalebox{.77}{
\begin{tabular}{|c|c|c|c|c|c|c|}
\hline
$\boldsymbol{\rho}$&$.5$ &$.6$ &$.7$ &$.8$ &$.9$ &$1$\\
\hline
\textbf{Test Loss}& 0.013 &0.013& 0.017& 0.017&  0.020&  0.017\\
\hline
\end{tabular}}
\caption{Test loss for approximate MLE using FW on the house data set with a gap of $50$ and different choices of uniform $\rho$ vectors.  }
\label{fig:rho}
\end{figure}

\begin{figure}[t!]
\centering
\includegraphics[width=.8\columnwidth]{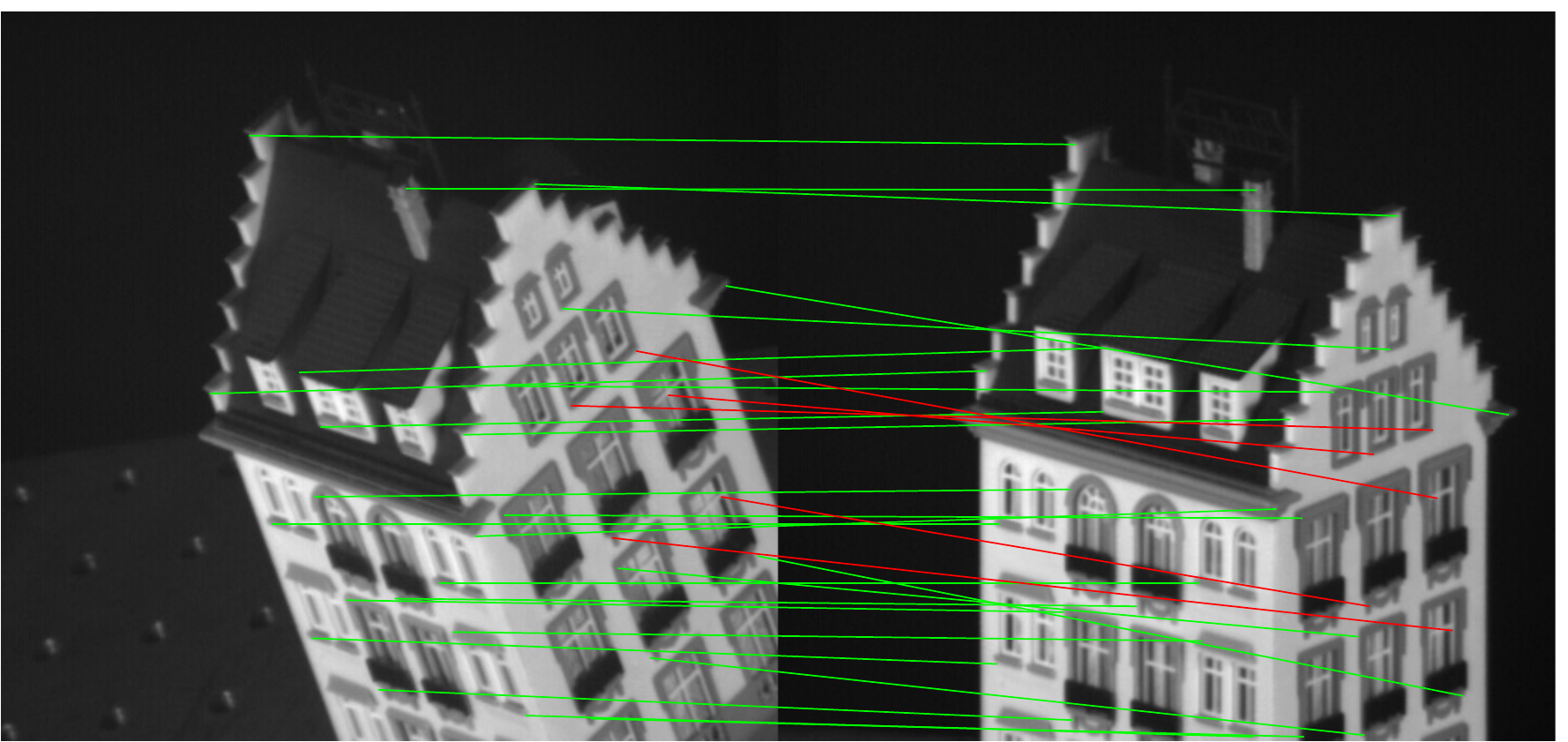}
\includegraphics[width=.8\columnwidth]{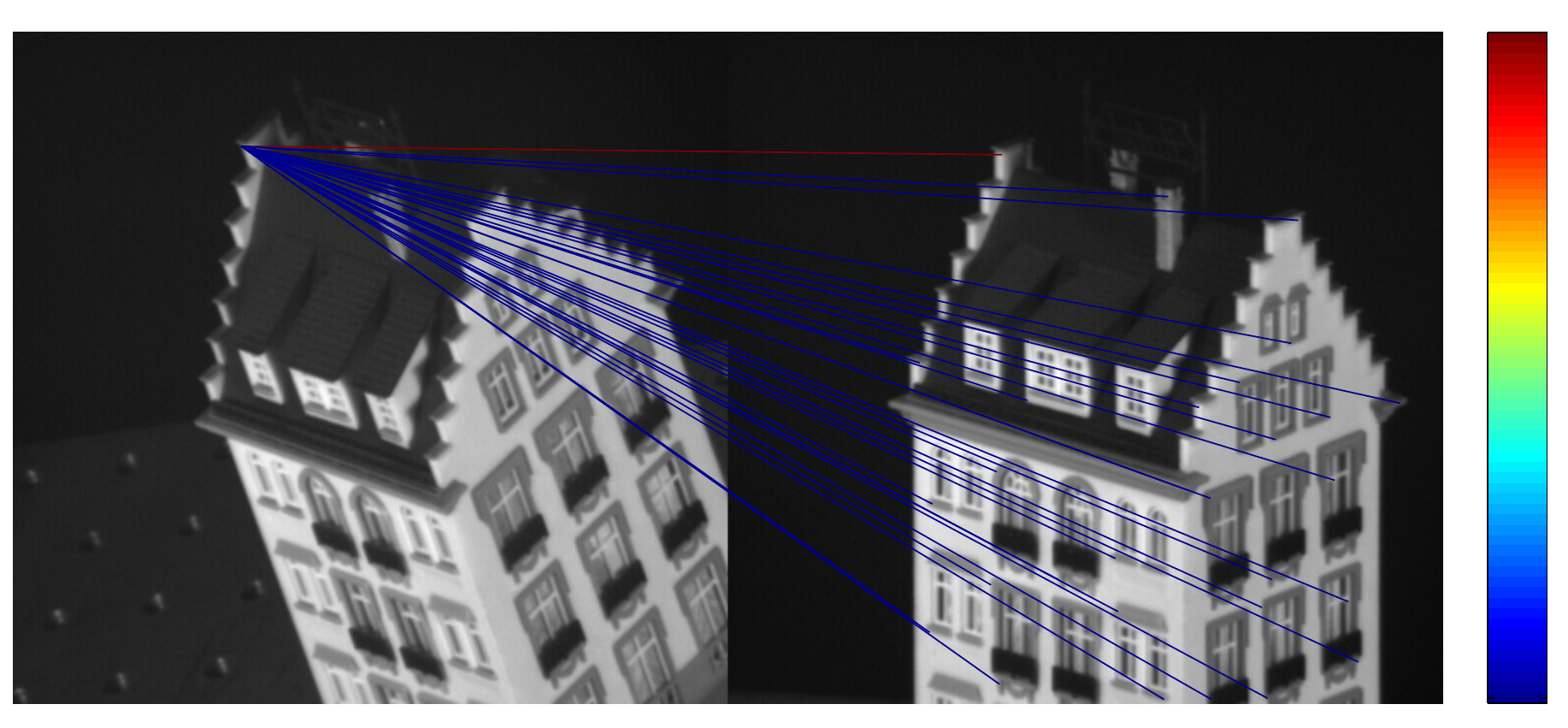}
\includegraphics[width=.8\columnwidth]{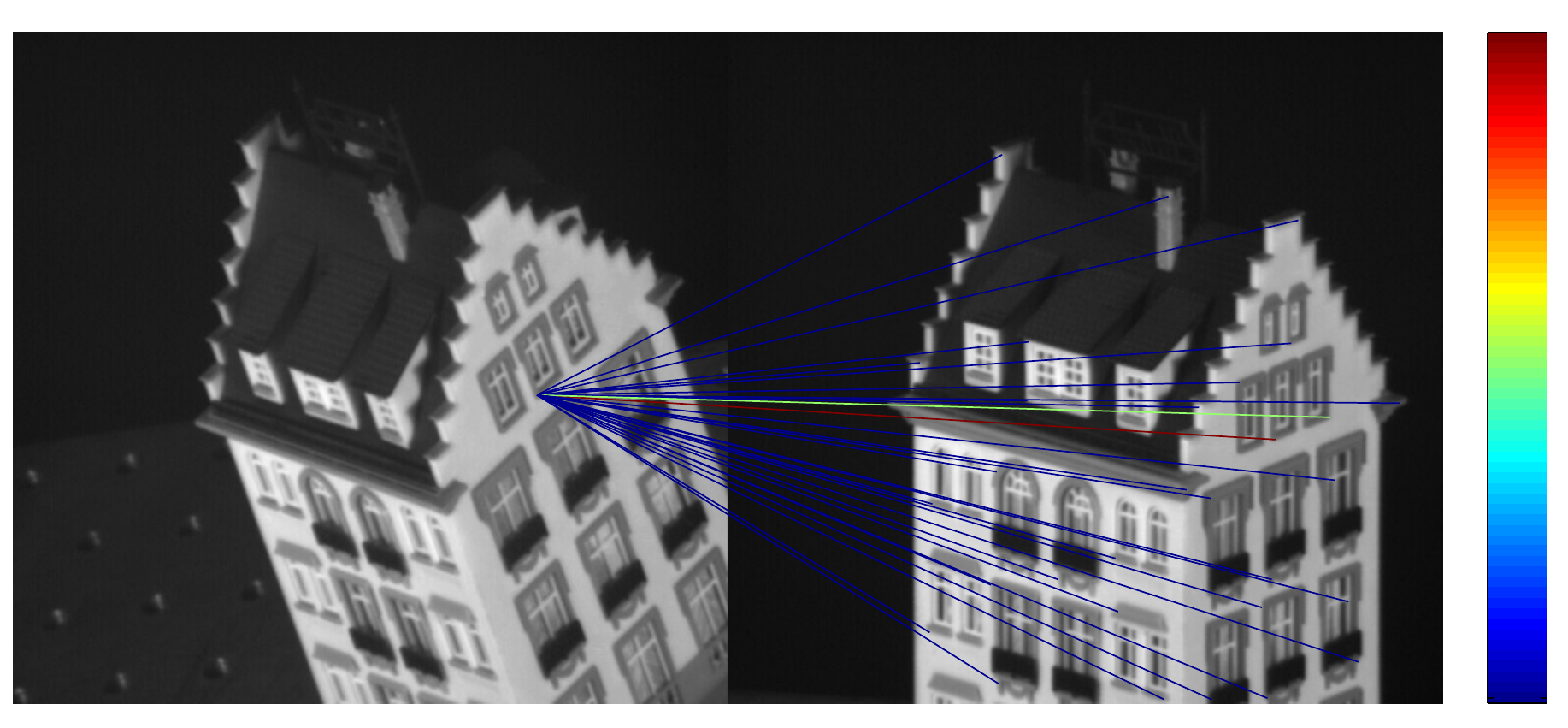}
\caption{
  \emph{Top:} Approximate MAP assignment produced from the learned model. Red edges were incorrectly matched.
  \emph{Middle:} Pseudomarginals for a correctly predicted edge. The correct edge has high probability (red) while all others have low probability (blue).
  \emph{Bottom:} Pseudomarginals for a wrongly predicted edge. There are two edges with nontrivial probability (red and green). When the model is forced to pick one, it picked the wrong one.
}
\label{fig:mismatch}
\end{figure}

\begin{figure}[t!]
\centering
\scalebox{.85}{
\begin{tabular}{|l|l|l|}
\multicolumn{3}{ c }{\textbf{Hotel}}\\
\hline
& \multicolumn{1}{c|}{\textbf{FW}}& \multicolumn{1}{c|}{\textbf{lin.+l.}}\\\hline
{ 0}&{0.0}&{0.0}\\\hline
{10}&{0.0}&0.0022\\\hline
{20}&{0.0049}&{0.0049}\\\hline
{30}&0.020&0.020\\\hline
{40}&0.023&{0.013}\\\hline
{50}&0.0614&{0.050}\\\hline
{60}&0.13&{0.12}\\\hline
{70}&0.17&{0.15}\\\hline
{80}&0.24&{0.19}\\\hline
{90}&0.33&{0.30}\\\hline
\end{tabular}
\begin{tabular}{|l|l|}
\multicolumn{2}{ c }{\textbf{House}}\\
\hline
\multicolumn{1}{|c|}{\textbf{FW}}& \multicolumn{1}{c|}{\textbf{lin.+l.}}\\\hline
{0}&{0.0}\\\hline
0.0040&{0.0}\\\hline
{0.0022}&{0.0022}\\\hline
0.0049&0.0\\\hline
{0.0}&{0.0}\\\hline
{0.017}&0.022\\\hline
{0.041}&0.051\\\hline
{0.051}&0.067\\\hline
{0.080}&0.14\\\hline
0.12&{0.11}\\\hline
\end{tabular}}

\caption{Test loss versus gap between frames for the hotel/house data measured via the Hamming loss for both MLE-Struct and the method of Caetano et al. (2009)}
\label{fig:hotels}
\end{figure}

We now apply the bipartite matching model to a graph matching problem arising in computer vision over the CMU \emph{house} and \emph{hotel} image sequences. We follow the setup of \citetmain{caetano2009learning}. The data consist of 111 frames of a toy house and 101 separate frames of a toy hotel, each rotated a fixed angle from its predecessor. Each frame was hand-labeled with the same 30 landmark points. We consider pairs of images at a fixed number of frames apart (the \emph{gap}), which we divide into training, validation, and testing sets following the same splits as \citetmain{caetano2009learning}. We measure the average Hamming error between the predicted matching (MAP estimate using our learned parameters) and the ground truth.

We compare our algorithm against the linear+learning method of \citetmain{caetano2009learning}, which fits the parameters of a linear model using the same features as our algorithm but with a hinge loss objective. The results of the experiments with reweighting parameters $\rho = \vec{1}$ are described in Figure \ref{fig:hotels}.  For each subsequence, we chose the regularization parameter via cross-validation.  Both methods perform comparably, with our method doing slightly better on the houses and the method of \citetmain{caetano2009learning} doing slightly better on the hotels. We also compared the performance of our algorithm with different reweighting parameters $\rho$.  Figure~\ref{fig:rho} shows the results for the house data when the gap is $50$ for various choices of $\rho$.  We observed little difference in test error as $\rho$ varies:  this was confirmed over synthetic as well as real data.  As a result, we did not tune $\rho$ for different data/problem setups.

Figure~\ref{fig:mismatch} illustrates one advantage of learning a probabilistic model over a discriminative model: the pseudomarginals indicate the model's confidence in a prediction. In many cases, when the algorithm made the wrong prediction, two edges incident to a specific node had relatively high pseudomarginal probabilities.  In these cases, the errors were not completely unfounded.  Similar image parts were matched, albeit incorrectly.

Algorithm~\ref{alg:fw-learning} permits fast and simple approximate MLE in problems where it was previously quite difficult. Namely, we found that using standard BP approaches~\citepmain{huang2009approximating,bayati2011} to compute marginals for the standard double-loop MLE approach was very unstable for this problem because BP often failed to converge after taking several gradient steps. In later sections, we juxtapose Algorithm~\ref{alg:fw-learning} with alternative approximate MLE approaches. 

\subsection{UNIPARTITE PERFECT MATCHINGS}
\begin{figure}[t!]
\centering
\includegraphics[width=\columnwidth]{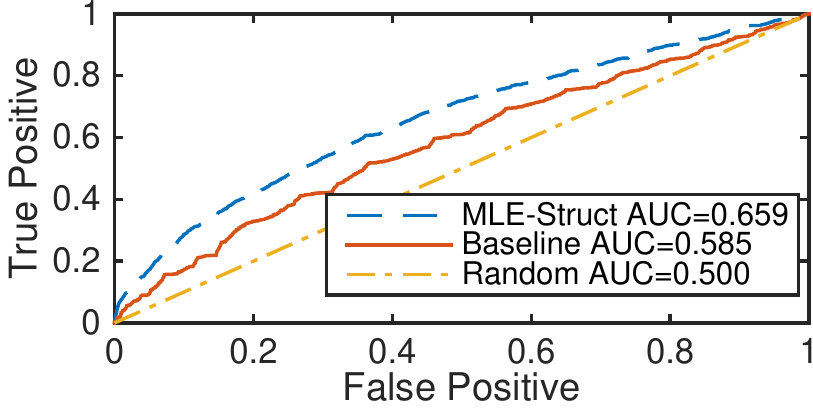}
\caption{ROC curves for roommate matching of our algorithm and a constant baseline.}
\label{fig:roommate-roc}
\end{figure}

\begin{table}
  \centering
  \begin{tabular}{|l|c|c|}
    \hline
    \textbf{Profile Item} &  \textbf{Weight} \\
    \hline
    Smoking	 & -0.0484 \\
    Personality& -0.0370 \\
    I generally go to bed at...	&	-0.0296\\
    I generally wake up at...	& -0.0218\\
    Study with audio/visual& -0.0133\\
    Overnight Guests& -0.0097\\
    Cleanliness	& -0.0056\\
    \hline
  \end{tabular}
  \caption{Largest $\ell_1$ distance features of roommate survey data.  More negative values increasingly discourage features from differing.}
  \label{tbl:roommates}
\end{table}

Many undergraduate institutions assign first-year students to roommates based on questionnaire responses, but allow  returning students to pick their own roommates in subsequent years. Therefore, we can use observed roommate matchings of returning students to train a model for students' preferences. Such a model can then be used by the administration to assign first-year students roommates that they will get along with. 

We obtained an anonymized dataset of campus housing room assignments and questionnaire responses for undergraduate students at a major US institution for the years 2010--2012. We used data from 2010 and 2011 to train and data from 2012 to test. We prune those students who did not live in campus housing or were assigned to single rooms (there were no rooms with three or more residents). The remaining students were thus assigned to one roommate, and form perfect matchings in complete graphs of 2374--2504 nodes. As our data includes neither year nor gender, we treat the entire matching assignment for one year as one observation.

Our questionnaire data consists of 2 binary features and 12 ordinal features of 5 levels each. For each pair of students and each questionnaire question, we created one feature of absolute differences and many interaction features, which consisted of one indicator feature for each possible pair of answers to the questionnaire questions. For simplicity, we assumed symmetric interactions. For each student pair, the weighted score for their matching is a linear combination of features.

We fit a model using MLE-Struct. Table~\ref{tbl:roommates} lists the largest coordinates of $\theta$ for distance features ordinal on ordinal questionnaire responses. Here, more negative values indicate closer agreement required. We see that smoking, personality (introverted vs extroverted), and sleeping habits require the strongest agreement. For more results and details of the feature, see Appendix~\ref{app:expt-details}.

As we are effectively performing multiclass classification of $\approx$ 2500 classes with only $\approx$ 14 features, we do not expect high accuracy in terms of Hamming error. Instead, we consider the use-case where we use the model to reject very bad roommate assignments. To evaluate this, we use our learned $\theta$ to form the cost matrix from features of the test year (2012), and use the entries of this cost matrix as scores for a binary classifier. We then plot ROC curves in Figure~\ref{fig:roommate-roc}, where we demonstrate gains above random guesses and a constant baseline where $\theta = -1$ for distance features and 0 for interactions. In particular, our algorithm dominates in the low false-positive regime of the graph. As competitors, we also evaluated a structured perceptron and structured SVM using the same MAP decoder, but even after extensive parameter tuning, they did not generalize well to test data, and obtained test AUCs worse than the constant baseline.

\subsection{GRID CRFS}
\begin{figure*}[t!]
  \centering
  \begin{subfigure}{\columnwidth}
    \includegraphics[width=\columnwidth]{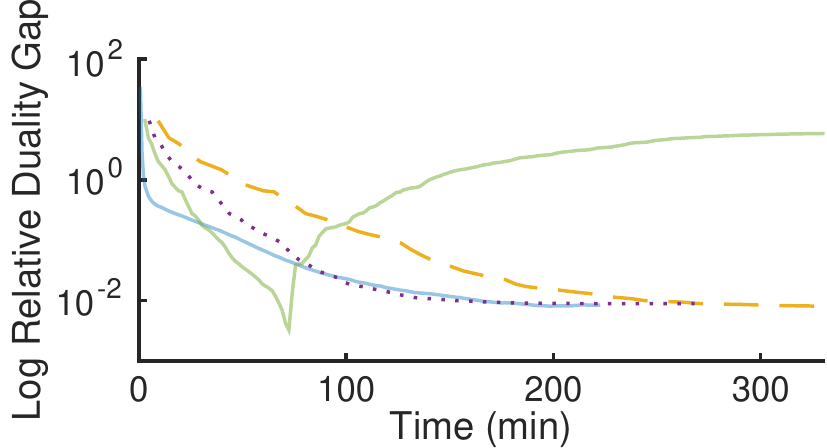}
    \caption{Plot of the difference to the midpoint of the best values attained by each algorithm (duality gap). The domke10 curve intersects, but does not converge to the correct value.}
  \end{subfigure}
  \begin{subfigure}{\columnwidth}
    \includegraphics[width=\columnwidth]{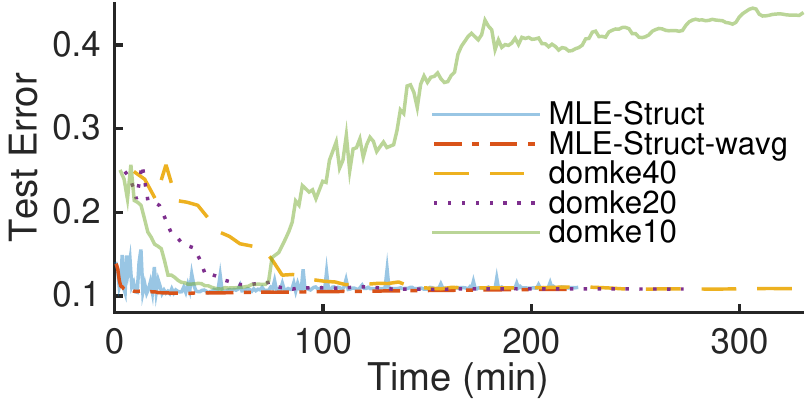}
    \caption{Test error (mean Hamming loss) vs time. MLE-Struct-wavg curve traces the error of our method with averaged iterates, which achieves the lowest test error and gets these fastest.}
  \end{subfigure}
\caption{Horse dataset performance measures. MLE-Struct is our method; legend is same across both plots.}
\label{fig:horseperm}
\end{figure*}

\begin{figure}[t!]
  \begin{subfigure}{\columnwidth}
    \centering
    \includegraphics[width=.39\columnwidth]{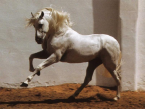}
    \includegraphics[width=.39\columnwidth]{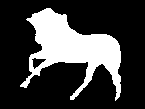}\\
    \caption{Raw test image (\#207) and ground truth segmentation.}
  \end{subfigure}
  \begin{subtable}{\columnwidth}
    \setlength{\tabcolsep}{-0.5pt}
    \begin{tabular}{ m{0.2\columnwidth} m{0.4\columnwidth} m{0.4\columnwidth} }

    \multicolumn{1}{m{0.2\columnwidth}}{} & \textbf{MLE-Struct} & \textbf{domke40}\tabularnewline
    9.2 min. &  \includegraphics[width=0.39\columnwidth]{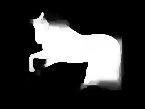} &  \includegraphics[width=0.39\columnwidth]{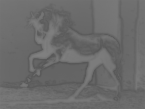}\tabularnewline
 
    3.7 hours & \includegraphics[width=0.39\columnwidth]{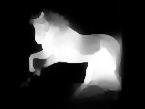} & \includegraphics[width=0.39\columnwidth]{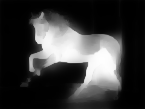}\tabularnewline

    \end{tabular}
    \caption{Estimated marginals of our algorithm and that of \citetmain{domke2013} on image \#207 after various amounts of training time.  Our algorithm terminated after 3.7 hours while that of \citemain{domke2013} took nearly twice as long to converge.}
  \end{subtable}
  \caption{Visual results of the horse dataset.}
  \label{fig:viz}
\end{figure}

Next we study a binary image segmentation problem on the Weizmann horses dataset \citepmain{borenstein2002class}. We formulate this as a pairwise binary model with a variable for each pixel indicating whether it is part of a horse. We used the same features and resized images obtained from \citetmain{domke2013} and kept their split of 200 training and 128 testing images. This results in grid CRFs of approximately 10,000--40,000 nodes.

In initial experiments, we tried a naive double-loop method of subgradient descent on $\theta$ using belief propagation to compute subgradients of the TRW log-partition function. This method was too slow and unstable for even small real problems. Instead we compared our algorithm against the method proposed by \citetmain{domke2013}, which solves \eqref{eq:saddle} by using a small, fixed number of iterations of TRW to approximate the partition function and thereby the gradient of of the approximate likelihood.  The outer loop runs L-BFGS until convergence.  Limiting the number of inner TRW iterations is key to this algorithm's efficiency, which burdens the user with another tuning parameter.

In Figure~\ref{fig:horseperm}, we compare two variants of our algorithm with three variants of the algorithm of \citetmain{domke2013} in terms of objective value and test error over time. Both methods optimize the same objective---we optimize the dual while they optimizes the primal. We set $\lambda = 3420$ in our parameterization which matches the setting for their published result. The MLE-Struct curves result from applying the block-coordinate version of our Algorithm~\ref{alg:fw-learning}. We obtained the fastest convergence without using linesearch. MLE-Struct-wavg uses the same algorithm, but evaluates the test error using a weighted average of the iterates as described by \citetmain{lacoste2013block}. The curve for averaged iterates is substantially smoother than the raw MLE-Struct curve, and very quickly attains a low test error. The domke$x$ curves result from running their algorithm for $x$ TRW inner loop iterations. This method is not guaranteed to converge to the global optimum for any finite $x$. A practitioner must run the algorithm for a sequence of increasing values of $x$ to confirm convergence to the correct value. In contrast, our FW algorithm requires only a single run and computes an upper-bound on the duality gap as a byproduct~\citepmain{jaggi2013revisiting}.

In early iterations our algorithm achieves the lowest objective value and test error. Our algorithm attains low test error even when the objective value is relatively far from optimal. This phenomenon is a result of the dual formulation: we iteratively move $\tau$ to minimize the objective, but for each value of $\tau$, we compute the optimal $\theta$ as a linear function of $\tau$. While $\tau$ may initially be very inaccurate, contributing to a large objective value, $\theta$ is much lower dimensional, so some of the errors may cancel in computing $\theta$, resulting in good predictions nevertheless. In contrast, the method of~\citetmain{domke2013} iteratively moves $\theta$, and for each value of $\theta$, computes the optimal value of $\tau$ using TRW. Prior to convergence, this may enable better fit to the training data at the expense of accurate estimation of $\theta$.

The effect is readily apparent when visualizing predicted marginals on the test set in Figure~\ref{fig:viz}. The domke40 method takes 9.2 minutes to complete one iteration. Parameters after one iteration are nowhere near the MLE, as evidenced by the first row of \ref{fig:viz}(b) and the early portion of the objective value plot \ref{fig:horseperm}(a); the mean Hamming loss on this sample is 0.249. Their marginal estimates at this point have only used local intensity data: light regions are classified as ``horse'' and dark regions are classified as ``not horse.'' In contrast, in about the same time, our BCFW method already ran for 12000 iterations, has made 60 passes over the training data, and essentially recovers the correct segmentation (except for difficult portions on the mane and hind legs, where the background texture is confusing) with mean Hamming loss of 0.068. After 3.7h, both algorithms produce comparable visualizations, though we get 0.074 normalized Hamming error on this sample, while domke40 gets 0.109. It takes 8.2h for the domke40 to converge according to its internal criteria, attaining a final test error of 0.0813 on this image.

\section{DISCUSSION AND CONCLUSION}
In maximum likelihood estimation of discrete exponential family models, replacing the Gibbs free energy with a convex free energy approximation leads to a concave-convex saddle point problem. We have shown that adding a quadratic regularizer enables a closed-form maximization, leaving a single convex minimization problem, which can be solved efficiently using the Frank-Wolfe algorithm. We can scale to large datasets by using block-coordinate Frank-Wolfe, and rapidly achieve low test error by solving the dual objective. This method accomplishes approximate MLE using a simple wrapper around a black-box MAP solver. Previously, practitioners either employed expensive double-loop MLE procedures or they abandoned MLE by resorting to structured SVMs and perceptrons. Our method is competitive with max-margin MAP-based estimation methods in terms of prediction error and faster than competing MLE methods in minimizing test error, while being simple to implement. In future work, we will extend the method to other combinatorial problems, incorporate structure learning with $\ell_1$ regularization, and handle latent variable models as a single minimax problem.

\bibliographystylemain{icml2015}
\bibliographymain{biblio}

\newpage
\onecolumn
\appendix
\begin{center}
\Large{Supplementary Material}
\end{center}

The following material is used to provide details about techniques employed in the paper. Part~\ref{app:expt-details} presents details of experimental setup and additional results. Then in part~\ref{app:fw} we discuss the convergence properties of FW for our problems. In part~\ref{app:conv} we prove convexity of the Bethe free energy for general matchings. Part~\ref{app:fwmatch} derives an instance of our Algorithm~\ref{alg:fw-learning} for matchings, and part~\ref{app:fw-crf} derives the same for linear CRFs.

Then, we present a full derivation of the dual objective and line search objective for doing FW learning for graph matchings. This is presented in terms of matrix-valued terms (various weighted adjacency matrices for the graph), which facilitates easy implementation, since MAP solvers operate on such matrices. 

\section{Details on Experiments}
\label{app:expt-details}

In this appendix, we explain further details of experiments.

\subsection{Synthetic Experiments}
\begin{figure}
  \centering
  \begin{subfigure}[b]{.49\columnwidth}
    \includegraphics[width=\columnwidth]{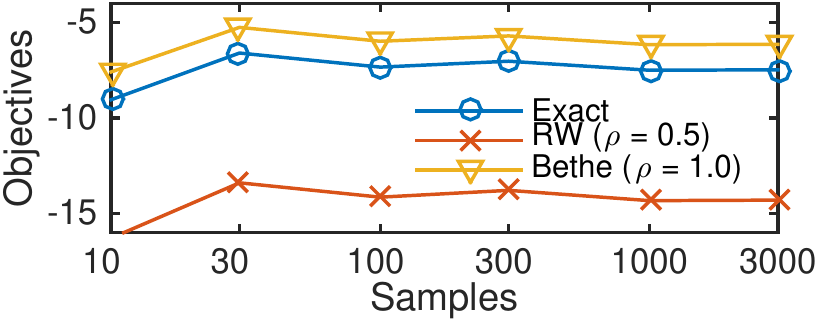}
    \includegraphics[width=\columnwidth]{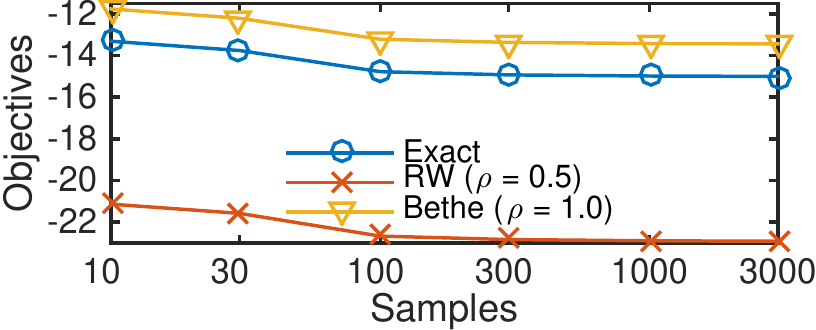}
    \caption{Exact and approximate likelihoods evaluated at the parameters that maximize the Bethe log-likelihood.}
  \end{subfigure}
  \begin{subfigure}[b]{.49\columnwidth}
    \includegraphics[width=\columnwidth]{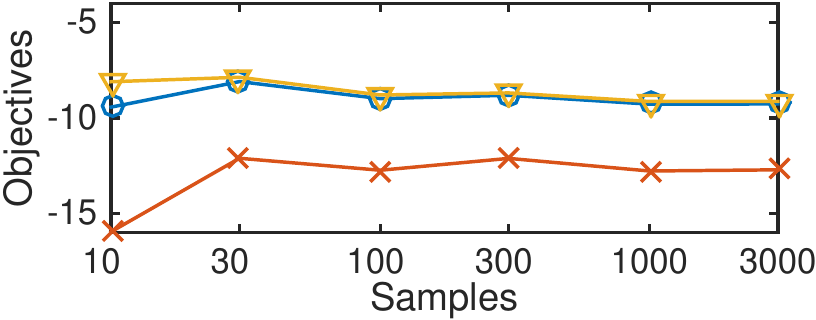}
    \includegraphics[width=\columnwidth]{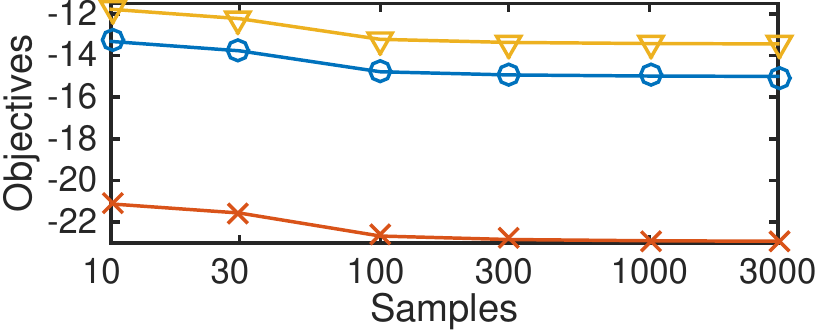}
    \caption{Exact and approximate likelihoods evaluated at the parameters that maximize the RW log-likelihood.}
  \end{subfigure}

  \caption{Sandwich bounds on the likelihood at the Bethe and TRW estimators. Top is the high SNR problem, bottom is the low SNR problem.}
  \label{fig:syn-cross-plots}
\end{figure}

For matchings, the TRW likelihood lower bounds the true likelihood, while the Bethe likelihood upper bounds the true likelihood at all parameter values. Thus, we can obtain two-sided bounds on the likelihood of $W^*_\text{TRW}$ by evaluating $\ell_B(W^*_\text{TRW})$, and of the likelihood of $W^*_B$ by evaluating $\ell_\text{TRW}(W^*_B)$. Figure~\ref{fig:syn-cross-plots} displays these results. The approximate likelihoods were computed using FW for inference, using the procedure described in Sec.~\ref{sec:fw-marg} while the exact likelihood was computed using Ryser's algorithm, which was feasible for this small problem.

\subsection{Horse Experiments}
Timing experiments were performed a dedicated 8 core (16 hyper-threads), 2.67GHz Intel Xeon X5550 machine with 24 GB of physical RAM, running Ubuntu 12.04.5 LTS and Matlab R2012b. Computations were restricted to a single core, and at most two experiments were run at a time. Our algorithms were implemented in Matlab, interfacing with combinatorial solvers written in C or C++. We downloaded the code for \citepappendix{domke2013} and \citepappendix{caetano2009learning} from the authors' websites, which were implemented in C++ and Matlab with C extensions, respectively. We obtained the original experiment scripts for \citepappendix{domke2013} through correspondence with the author. 

\subsection{Roommate Experiments}

This dataset was obtained from a major US university over a three year period from 2010-2012.  The anonymized dataset consists of roommate assignments for pairs of students in each of the three years.  In addition, each of the students was required to complete a brief housing survey that asked for their preferences in terms of cleanliness, sleeping schedule and habits, personality, study preferences, etc.  Our questionnaire data consists of 2 binary features and 12 ordinal features of 5 levels each. For each pair of students and each questionnaire question, we created one feature of absolute differences and several interaction indicator features, one for each possible pair of answers to the questionnaire questions. For simplicity, we assumed symmetric interactions.  For each student pair, the weighted score for their matching is a linear combination of features.   The learned weights for each \emph{distance} features and their relative rankings are described in Figure \ref{fig:roommates}.  As we are using a log-linear model, the weights should be interpreted as log-odds ratio for a unit increase in absolute distance (assuming ordinal features).

A few qualitative observations about these results.  First,  \emph{single-sex floor}, \emph{rising sophomore}, and \emph{allow in Brownstone} all received relatively low weights, indicating perhaps that the data was too noisy with respect to these survey responses.  Second,  \emph{personality}, \emph{smoking}, and \emph{bedtime} were among the strongest predictors of a successful match while \emph{cleanliness}, \emph{study hours}, \emph{study location} were among the least important. 

When comparing to the BCFW algorithm for a structured SVM, we employed the publicly-available code from the authors. We tried regularization parameter lambda values in the range [10e-4, 10e2]. Our best-performing configuration achieved an AUC of .504 and a hamming error of .9992, which outperforms random guessing, but significantly underperforms a model trained with MLE. 
\begin{figure}
\centering
\begin{tabular}{|l|c|c|}
\hline
Profile Item Description	&  Weight	& Rank \\
\hline
I generally go to bed at...	&	-0.0296	& 3\\
I generally wake up at...	& -0.0218 &	4\\
Rising Sophomore	& 0 &	10\\
Cleanliness	& -0.0056 &	7\\
Smoking	 & -0.0484 &	1\\
Sleeping Habits & -0.0025 &	8\\
Overnight Guests& -0.0097 & 6\\
Personality& -0.037 &2\\
Usual Study Hours&0.0131 &	12\\
Study Location&-0.0006&	9\\
Study with audio/visual& -0.0133 &	5\\
Single-sex floor (1)&0.2277 &14\\
Single-sex floor (2)& 0.0518 & 13\\
Allow in Brownstone&0&10\\
\hline
\end{tabular}
\caption{Distance features for the roommates experiments, their learned weights, and their relative importance ranking with regularization $\lambda=100$.}
\label{fig:roommates}
\end{figure}

\section{Convergence of Frank-Wolfe}
\label{app:fw}
Recall the following approximate max-entropy objective function.
\begin{align*}
L(\tau^{(1:m)}) = \frac{1}{2\lambda}\|\theta^*(\tau^{(1:M)})\|^2 - \sum_m H_\rho(\tau^{(m)}),
\end{align*}
In this appendix, we discuss the convergence rate of the FW algorithm for the minimization of the convex function $L$ over the local polytope.  \citetappendix{jaggi2013revisiting} has shown that the suboptimality of the iterates of FW decays as 
\begin{equation}
\label{eq:convergence}
L(\tau_t) - L(\tau^* ) \leq \frac{2 C_L}{t + 2}(1 +
\delta),
\end{equation}
where $\delta$ is the accuracy to which each of the linear subproblems (i.e., the optimization problem performed at each iteration) is solved and $C_L$ is the \textit{curvature} of the function $L$.

Curvature is a stronger notion of the function's geometry than its
Lipschitz parameter, since it is affine-invariant, like the entire
Frank-Wolfe algorithm~\citepappendix{jaggi2013revisiting}.  The curvature of a differentiable function $F:X \rightarrow \mathbb{R}$ is given by
\begin{equation}
\label{eq:curvature}
C_F = \sup_{\begin{subarray}{c}
        x,x' \in X, \; \gamma \in [0,1] \\ y = x + \gamma (x' - x)
      \end{subarray}}
 \frac{2}{\gamma^2}\left(F(y) - F(x) - \dt{y - x}{\nabla F(x)}\right).
\end{equation}
The curvature, $C_F$, quantifies how much $F$ can differ from its linearization.  For twice differentiable functions, the curvature can be upper bounded as follows.
\begin{align*}
C_F \leq \sup_{\begin{subarray}{c}x,x'\in X, \; \gamma \in [0,1] \\ y = x + \gamma (x' - x)\end{subarray}} \frac{1}{2} (y-x)^T\nabla^2 F (y-x)
\end{align*}
For our objective function, the curvature is most heavily influenced by the entropy term as the curvature of the quadratic piece is simply a constant that depends on the $\phi$'s.  Unfortunately, the curvature of $L$ is unbounded: as one approaches integer points in the local polytope, the entropy approximation becomes arbitrarily steep. Therefore, the worst-case convergence rate of FW for this problem is unbounded.  In order to obtain a bounded curvature, we could strengthen the box constraints in the marginal polytope by requiring
\begin{align*}
\tau_i(y_i) \in [\eta, 1-\eta]& \text{ for all } i\in V, y_i\in\mathcal{Y}\\
\tau_\alpha(y_\alpha) \in [\eta, 1-\eta]& \text{ for all } \alpha\in\mathcal{A}, y_\alpha\in\mathcal{Y}^{|\alpha|}
\end{align*}
for some $\eta \in (0,.5)$. For each $\tau_i(x_i)$, the part of the entropy approximation that depends on $\tau_i(x_i)$ looks like
\[\left(1 - \sum_{\alpha\supset i} \rho_\alpha\right)\tau_i(x_i)\log \tau_i(x_i).\]
That is, all we need to do is compute bounds on the curvature for functions of the form $f(x) = c\cdot x\log x$.  For this function,
\begin{align*}
C_f \leq \frac{c(x'-x)^2}{2y}
\end{align*}
for any $y\in[x, x']$.  To bound this quantity from above, we can pick $x = \eta$, $x' = 1 -\eta$, and $y =\eta$.  This gives
\[C_f \leq \frac{c(1-2\eta)^2}{2\eta},\]
which is a fixed constant and tends to zero as $\eta$ tends towards $.5$.  Hence, as long as the optimal pseudomarginals lie strictly inside $[0,1]$, then this modified FW is always guaranteed a linear rate of convergence to the optimum.  The pseudomarginals produced by both BP and RBP often lie strictly inside the box constraints, so this is typically not an issue in practice.  In order to find an appropriate $\eta$, we could use, for example, the bounds on the pseudomarginals proposed by \citetappendix{mooij2007}.   These bounds are obtained by running BP/RBP for a fixed number of iterations.  An appropriate $\eta$ can then be selected such that the interval $[\eta, 1-\eta]$ contains all of the bounds as in \citetappendix{mooij2007}.

Adding additional box constraints is expensive (we can no longer use the combinatorial algorithms that we used for matching and pairwise binary MRFs), and they may not be necessary in practice. For pseudomarginals $\tau$, the steepness only becomes unmanageable when there are components of $\tau$ that are close to 0 or 1 (in a Gibbs distribution, no clique assignment ever has zero probability, so
$\nabla L$ is always well-defined). If the iterates of the algorithm never get too close to the boundary of $\mathcal{T}$, then the effective curvature term will be reasonable. Of course, if the optimal pseudomarginals of the
reweighted approximation are close to the boundary of
$\mathcal{T}$, then $C_L$ will be large in the neighborhood of the
solution, and we should not expect fast convergence. A rough way to estimate this distance from the boundary is via the entropy of the pseudomarginals, and our experience has shown that the algorithm converges faster when the true pseudomarginal distribution has higher entropy.  

\section{Convexity of the Bethe Free Energy for General Matchings}
\label{app:conv}
In this appendix, we argue that the Bethe free energy for the (not necessarily perfect) matching problem is convex over general graphs.  The convexity of the Bethe approximation for the bipartite matching problem was investigated experimentally by \citetappendix{huang2009approximating} and the proven by \citetappendix{betheperm}.   The same argument holds for any choice of reweighting parameters such that $\rho_i\in[0,1]$ for all $i$.  For simplicity we only argue the case $\rho_i = 1$ for all $i$, the general case is similar to Theorem 60 in \citetappendix{betheperm}.  The entropy and polytope approximations are formulated as follows.
\begin{align*}
H'_\rho(\tau) \triangleq &\sum_{(i,j)\in E} \Big[(\rho_i + \rho_j - 1)(1-\tau_{ij})\log(1-\tau_{ij}) - \tau_{ij}\log\tau_{ij} \Big]\\
\:& - \sum_{i\in V} \rho_i(1-\sum_{j\in\partial i} \tau_{ij})\log (1-\sum_{j\in\partial i} \tau_{ij})
\end{align*}
where $\tau$ is restricted to $\mathcal{T}' = \{\tau\geq 0 : \mbox{for all } i \in V, \sum+{j\in\partial i} \tau_{ij} \leq 1\}$.

The proof of convexity follows directly from the concavity of the function 
\[S_n(x_1,\ldots,x_n) = \sum_{i=1}^n (1-x_i)\log (1-x_i) - x_i\log x_i \]
for each $n \geq 1$ \citepappendix{betheperm}.

\begin{theorem}
For any $\rho \in [0,1]^{|V|}$, any graph (bipartite or general), and any matching (perfect or imperfect), the reweighted free energy~\eqref{eq:rfe} is convex over the local polytope.
\end{theorem}
\begin{proof}
For the case of the perfect matching problem on a graph $G$, the entropy term of the Bethe free energy can be written as
\begin{align*}
H'_{\vec{1}}(\tau) = & \Big[\sum_{(i,j)\in E} (1-\tau_{ij})\log(1-\tau_{ij}) - \tau_{ij}\log\tau_{ij}\Big] - \sum_{i\in V} (1-\sum_{j\in\partial i} \tau_{ij})\log (1-\sum_{j\in\partial i} \tau_{ij})\\
= & \sum_{i\in V} \Big[-(1-\sum_{j\in\partial i} \tau_{ij})\log (1-\sum_{j\in\partial i} \tau_{ij}) + \frac{1}{2}\sum_{j\in\partial i} \Big((1-\tau_{ij})\log(1-\tau_{ij}) - \tau_{ij}\log\tau_{ij}\Big)\Big]\\
= & \sum_{i\in V}\Big[ \frac{1}{2}S(\tau_{i,\partial i}, 1-\sum_{j\in\partial i} \tau_{ij}) + \frac{1}{2}h(1-\sum_{j\in\partial i} \tau_{ij})\Big].
\end{align*}
Here, $h(x) = -x\log x - (1-x)\log(1-x)$ is the entropy function.  As both $S$ and $h$ are concave functions, the entropy function is concave which implies that the free energy approximation is convex.
\end{proof}

\section{Frank-Wolfe and Matchings}
\label{app:fwmatch}
In this appendix, we describe a conditional random field over perfect matchings, formulate the approximate learning problem in this context, and describe the linesearch procedure used as part of the FW algorithm.

Assume we have $M$ observations, consisting of $N$ items matched to $N$
other items. We represent the $m$'th observation by $\left(W^{(m)},X^{(m)},Y^{(m)}\right)$
where the $W$ and $X$ are $N\times D_{W}$ and $N\times D_{X}$
data matrices and $Y^{(m)}$ is an $N\times N$ column permutation
matrix
\footnote{That is, if $i$ maps to $j$, then $Y_{ji}=0$ and $Y_{ki}=0$ for
$k\neq j$.
}. Note that $W$ and $X$ contain the data for the two separate parts
of the graph.

In general, conditional random field features can be arbitrary functions
of $\left(W,X,Y\right)$. To produce a model whose MAP solution is a
maximum-weight perfect matching, we require the features to be linear
in $Y$. Since $Y_{ji}$ denotes the presence
or absence of edge $(i,j)$, its coefficient ought to depend only
on the data for items $i$ and $j$. Therefore, we use the feature
map $F_{k}(W,X,Y)=\left\langle G_{k}(W,X),Y\right\rangle $ where
$\left(G_{k}(W,X)\right)_{ij}=g_{k}\left(w_{i},x_{j}\right)$. That
is, the $k$'th feature is a linear function of $Y$ with coefficients
given by applying a single function $g_{k}:\mathbb{R}^{D_{X}}\times\mathbb{R}^{D_{W}}$
to every pair of rows in $W$ and $X$. We will have $K$ features
in total. We now write $G_{k}^{(m)}=G_{k}\left(W^{(m)},X^{(m)}\right)$
and dispense with $W$ and $X$. The probability of one observation
is thus
\[
p(Y|G_{1:K};\theta)=\frac{1}{Z(\theta)}\exp\left(\sum_{k}\theta_{k}\left\langle G_{k},Y\right\rangle \right)
\]
So the log-likelihood for $M$ i.i.d. observations is
\begin{equation}
\ell\left(\theta;Y^{(1:M)},G_{1:K}^{(1:M)}\right)=\sum_{m}\sum_{k}\theta_{k}\left\langle G_{k}^{(m)},Y^{(m)}\right\rangle -\log Z\left(\theta,G_{1:K}^{(m)}\right)\label{eq:loglike}
\end{equation}
We focus on the case $K\leq MN^{2}$.

\subsection{Minimax Formulation}

We now replace $\log Z$ with $\log Z_{B,\rho}$ and add an $L_{2}$
regularizer. Note that $\rho$ is an $N$-vector reweighting parameter
with entries between 0.5 and 1. Using the variational formulation
of the (Bethe) free energy, we write the maximum Bethe likelihood
problem as a minimax problem, which we further analytically reduce
to a convex program with linear constraints. Begin with
\begin{equation}
-\log Z_{B}\left(\theta,G_{1:K}^{(m)}\right)=\min_{T\in{\cal M}}-\sum_{k}\theta_{k}\left\langle G_{k}^{(m)},T\right\rangle -H_\rho(T).\label{eq:bfe}
\end{equation}
To simplify subsequent derivations, let $y^{(m)}=\vvec\left(Y^{(m)}\right)$,
$\tau^{(m)}=\vvec\left(T^{(m)}\right)$, and $G^{(m)}$ be an $N^{2}\times K$
matrix whose $k$'th column is given by $\vvec\left(G_{k}^{(m)}\right)$.
In the sequel, we will use the reweighting parameters in the form
of pairwise sums $\rho_{i}+\rho_{j}$. Thus, let $R$ be $N\times N$
matrix where $R_{ij}=\rho_{i}+\rho_{j}$ and let $r=\vvec(R)$. Additionally,
define $y,\tau,$ and $G$ by vertically stacking all $M$ members
of $y^{(m)},\tau^{(m)},$ and $G^{(m)}$. Thus we can rewrite $\sum_{m}\theta_{k}\left\langle G_{k}^{(m)},Y^{(m)}\right\rangle =\theta^{\top}\left(G^{\top}y\right)$.
Plugging (\ref{eq:bfe}) into (\ref{eq:loglike}) and adding a quadratic
penalty gives the problem
\begin{eqnarray}
 &  & \max_{\theta}\theta^{\top}\left(G^{\top}y\right)-\frac{\lambda}{2}\left\Vert \theta\right\Vert _{2}^{2}+\sum_{m}\min_{\tau^{(m)}\in{\cal M}}-\theta^{\top}\left(G^{(m)\top}\tau^{(m)}\right)-H_\rho\left(\tau^{(m)}\right)\nonumber \\
 & = & \max_{\theta}\theta^{\top}\left(G^{\top}y\right)-\frac{\lambda}{2}\left\Vert \theta\right\Vert _{2}^{2}+\min_{\tau\in{\cal M}^{M}}-\theta^{\top}\left(G^{\top}\tau\right)-\sum_{m}H_\rho\left(\tau^{(m)}\right)\nonumber \\
 & = & \min_{\tau\in{\cal M}^{M}}\max_{\theta}\theta^{\top}\left(G^{\top}(y-\tau)\right)-\frac{\lambda}{2}\left\Vert \theta\right\Vert _{2}^{2}-\sum_{m}H_\rho\left(\tau^{(m)}\right)\label{eq:minimax1}\\
 & =: & \min_{\tau\in{\cal M}^{M}}\max_{\theta}f\left(\tau,\theta\right)\nonumber 
\end{eqnarray}
The second line is justified because the minimizations in $\tau^{(m)}$
are separable, so the $\min$ and sum operators commute. The cost
is that we must now minimize over a larger product space ${\cal M}^{M}$,
but we will see later why this is not a problem. The last line follows
from Sion's minimax theorem: the minimization domain ${\cal M}^{M}$
is compact convex, and the objective is convex in the minimization
variable $\tau$ and concave in the maximization variable $\theta$
\citepappendix{sion1958}. The theorem requires only \emph{one} compact domain,
so $\theta$ can remain unconstrained. Thus, for any $\tau$, the
concave function $f\left(\tau,\cdot\right)$ attains its maximum at
the stationary point $0=\nabla_{\theta}f=G^{\top}(y-\tau)-\lambda\theta$,
e.g. $\theta=\lambda^{-1}G^{\top}(y-\tau)$. Moreover, $f\left(\tau,\cdot\right)$
is \emph{strictly }concave for $\lambda>0$\emph{,} so the maximum
is unique. Plugging in to (\ref{eq:minimax1}) and simplifying, we
get
\begin{eqnarray}
 &  & \min_{\tau\in{\cal M}^{M}}\frac{1}{2\lambda}\left\Vert G^{\top}(y-\tau)\right\Vert _{2}^{2}-\sum_{m}H_\rho\left(\tau^{(m)}\right)\label{eq:obj1}\\
 & =: & \min_{\tau\in{\cal M}^{M}}h\left(\tau\right)\nonumber 
\end{eqnarray}

\subsection{Line search}
\label{app:line-search}
To compute the next iterate of FW, $\tau_{t+1}$, we use linesearch. Write $\tau_{t+1}=(1-\tau)\tau_{t}+\eta\tau_{t+1}^{*}$.
Plugging in to (\ref{eq:obj1}), we get
\begin{eqnarray}
h_{t}(\eta) & := & \frac{1}{2\lambda}\left\Vert G^{\top}\left(y-(1-\eta)\tau_{t}-\eta\tau_{t+1}^{*}\right)\right\Vert ^{2}-\sum_{m}H_{RW}\left((1-\eta)\tau_{t}^{(m)}+\eta\tau_{t+1}^{*(m)};\rho\right)\nonumber \\
 & = & \frac{1}{2\lambda}\left\{ \left\Vert G^{\top}\left(y-\tau_{t}\right)\right\Vert ^{2}+2\eta\left(y-\tau_{t}\right)^{\top}GG^{\top}\left(\tau_{t}-\tau_{t+1}^{*}\right)+\eta^{2}\left\Vert G^{\top}\left(\tau_{t}-\tau_{t+1}^{*}\right)\right\Vert ^{2}\right\}\nonumber\\
 & & -\sum_{m}H_{RW}\left((1-\eta)\tau_{t}^{(m)}+\eta\tau_{t+1}^{*(m)};\rho\right)\label{eq:linesearch}
\end{eqnarray}
Thus we can precompute the expensive matrix products in the quadratic
term.

\subsection{General Matchings}
The above FW procedure only requires a few changes when switching from complete bipartite graphs to general graphs.
The same equations and steps hold when we replace biadjacency feature matrices with adjacency features, and permutation matrices with matrices representing perfect matchings.
There are only a few technical caveats.
First, for general graphs we need to be able to allow some $\tau_{ij}$ to be zero.
This can occur because either there is no edge between $i$ and $j$, or $i$ and $j$ are neighbors, but there is no possible perfect matching in which they are linked.
For both of these cases, we simply clamp $\tau_{ij}$ at zero.
Similarly, some edges occur in every perfect matching, so we need to discover these a-priori and clamp $\tau_{ij}$ at one.
Second, unlike for bipartite matching, initialization of $\tau$ is non-trivial, since the set of neighbors $\text{Nb}(i)$ is different for every $i$.
We cannot choose an integral $\tau$ from the local marginal polytope as an initial point, since the curvature is infinite there.
Instead, for every edge in the graph, we can find one matching that contains that edge and one matching does not by solving a series of matching problems.
We average all of these matchings to obtain an initial feasible point.

\section{Frank Wolfe and Linear CRFs}
\label{app:fw-crf}

\subsection{Notation}

We work with a conditional random field of $L$ labels over the graph
$G=(N,E)$ in the standard overcomplete parameterization. That is,
$y_{n}$ is an $L\times1$ indicator vector for the state of node
$v$, and $y_{e}$ is an $L\times L$ indicator matrix for the state
of edge $e$. We will also treat $y_{e}$ as a vector when convenient.
We denote an element of a matrix or vector by parentheses. For node
$n$, let $u_{n}$ be its $C\times1$ feature vector and for edge
$e$, let $v_{e}$ be its $D\times1$ feature vector. Implicitly,
these feature vectors are derived from applying some function to an
input vector $x$. We refer to elements of a vector We will learn
a linear map for the node and edge parameters:
\begin{eqnarray*}
\theta_{n} & = & Fu_{n}\quad\forall n\in N\\
\theta_{e} & = & Gv_{e}\quad\forall e\in E
\end{eqnarray*}
Now suppose we have $M$ exchangeable samples, and let the superscript
$\cdot^{m}$ denote the observation belonging to the $m$th sample.
Our joint log-likelihood is thus
\[
\ell\left(F,G;y,u,v\right)=\sum_{m}\left(\sum_{n}y_{n}^{m\top}Fu_{n}+\sum_{e}y_{e}^{m\top}Gv_{e}\right)-\log Z\left(F,G,u,v\right)
\]

\subsection{Minimax Formulation}

We replace $\log Z$ with a parameterized surrogate likelihood $\log Z_{\rho}$
which interpolates between the TRW and Bethe approximations. We use
the variational formulation of $\log Z_{\rho}$, over the \emph{local}
polytope ${\cal T}$. Note that the Bethe approximation is not convex
in this setting, but TRW is. For grid MRFs, each edge has probability $0.5$ of appearing in a
spanning tree, so we set $\rho=0.5$ for each edge.

Since we are estimating matrix parameters, we add a Frobenius penalty.
The minimax formulation is thus
\begin{eqnarray}
 &  & \max_{F,G}\sum_{m}\left(\sum_{n}y_{n}^{m\top}Fu_{n}^{m}+\sum_{e}y_{e}^{m\top}Gv_{e}^{m}\right)-\frac{\lambda}{2}\left\Vert F\right\Vert _{F}^{2}-\frac{\lambda}{2}\left\Vert G\right\Vert _{F}^{2}\nonumber \\
 &  & +\sum_{m}\min_{\mu^{m}\in{\cal T}}-\left(\sum_{n}\mu_{n}^{m\top}Fu_{n}^{m}+\sum_{e}\mu_{e}^{m\top}Gv_{e}^{m}\right)-H_{\rho}(\mu)\nonumber \\
 & = & \min_{\mu\in{\cal T}^{M}}\max_{F,G}\sum_{m}\left(\sum_{n}\left(y_{n}^{m}-\mu_{n}^{m}\right)^{\top}Fu_{n}^{m}+\sum_{e}\left(y_{e}^{m}-\mu_{e}^{m}\right)^{\top}Gv_{e}^{m}\right)\label{eq:minimax}\\
 &  & -\frac{\lambda}{2}\left\Vert F\right\Vert _{F}^{2}-\frac{\lambda}{2}\left\Vert G\right\Vert _{F}^{2}-\sum_{m}H_{\rho}\left(\mu^{(m)}\right)\nonumber 
\end{eqnarray}
where the reweighted approximate entropy is given by
\begin{eqnarray}
H_{\rho}(\mu) & := & \sum_{n\in N}H(\mu_{n})-\sum_{nn'\in E}\rho_{nn'}I(\mu_{nn'})\nonumber \\
 & = & \sum_{n\in N}H(\mu_{n})-\sum_{n\in N}\sum_{n'\in\Ne(n)}\rho_{nn'}\left[H(\mu_{n})+H(\mu_{n'})\right]+\sum_{nn'\in E}\rho_{nn'}H(\mu_{nn'})\nonumber \\
 & = & \sum_{n\in N}\left(1-\sum_{n'\in\Ne(n)}\rho_{nn'}\right)H(\mu_{n})+\sum_{nn'\in N}\rho_{nn'}H(\mu_{nn'})\label{eq:entropy}
\end{eqnarray}
where $I(\mu_{nn'})=\sum_{y_{n}y_{n'}}\mu_{nn'}(y_{n},y_{n'})\log\left[\mu_{nn'}(y_{n},y_{n'})/\mu_{n}(y_{n})\mu_{n'}(y_{n'})\right]$
is the mutual information between variables $n$ and $n'$ and $H(\mu_{n})=-\sum_{y_{n}}\mu_{n}(y_{n})\log\mu_{n}(y_{n})$
and $H(\mu_{nn'})=-\sum_{y_{n}y_{n'}}\mu_{nn'}(y_{n},y_{n'})\log\mu_{nn'}(y_{n},y_{n'})$
are singleton and pairwise entropies. We have used the identity $I(\mu_{nn'})=H(\mu_{n})+H(\mu_{n'})-H(\mu_{nn'})$.
We have implicitly used the \emph{pairwise} marginalization constraints
when using the mutual information identity, so these gradients are
valid only on the local polytope---a fact that is important to remember
when optimizing.

The stationary point of the objective in~\eqref{eq:minimax} is thus
\begin{eqnarray}
0 & = & \sum_{mn}\left(y_{n}^{m}-\mu_{n}^{m}\right)u_{n}^{m\top}-\lambda F\nonumber \\
\Rightarrow F & = & \lambda^{-1}\sum_{mn}\left(y_{n}^{m}-\mu_{n}^{m}\right)u_{v}^{m\top}\label{eq:F}
\end{eqnarray}
Similarly, $G=\lambda^{-1}\sum_{me}\left(y_{e}^{m}-\mu_{e}^{m}\right)v_{e}^{\top}$.
Recalling the definition of the Frobenius norm and rearranging some
summations, we get
\begin{eqnarray}
\frac{\lambda}{2}\left\Vert F\right\Vert _{F}^{2}=\frac{\lambda}{2\lambda^{2}}\left\Vert \sum_{mn}\left(y_{n}^{m}-\mu_{n}^{m}\right)u_{n}^{m\top}\right\Vert _{F}^{2} & = & \frac{1}{2\lambda}\sum_{mn}\sum_{m'n'}\left(y_{n}^{m}-\mu_{n}^{m}\right)^{\top}\left(y_{n'}^{m'}-\mu_{n'}^{m'}\right)u_{n'}^{m'\top}u_{n}^{m}\label{eq:Fnorm}
\end{eqnarray}
On the other hand, the quadratic terms are
\begin{eqnarray}
\frac{1}{\lambda}\sum_{mn}\left(y_{n}^{m}-\mu_{n}^{m}\right)^{\top}\sum_{m'n'}\left(y_{n'}^{m'}-\mu_{n'}^{m'}\right)u_{n'}^{m'\top}u_{n}^{m} & = & \frac{1}{\lambda}\sum_{mn}\sum_{m'n'}\left(y_{n}^{m}-\mu_{n}^{m}\right)^{\top}\left(y_{n'}^{m'}-\mu_{n'}^{m'}\right)u_{n'}^{m'\top}u_{n}^{m}\nonumber \\
 & = & 2\lambda\left\Vert F\right\Vert _{F}^{2}\label{eq:quad}
\end{eqnarray}
e.g. the same Frobenius norm of outer products, by comparison with
(\ref{eq:Fnorm}).

We have eliminated the matrix $F$ (and similarly, $G$), which reveals
that our objective is a quadratic form in the Gram matrices $UU^{\top}$
and $VV^{\top}$, where $U$ is obtained by vertically stacking the
$u_{n}^{m}$ and $V$ obtained by vertically stacking the $v_{e}^{m}$,
so that the entry $\left(UU^{\top}\right)_{nm,n'm'}=u_{n}^{m\top}u_{n'}^{m'}$
and $\left(VV^{\top}\right)_{em,e'm'}=v_{e}^{m'\top}v_{e'}^{m'}$.
Let $Y_{N},T_{N}$ be the matrices obtained whose $(mn,\ell)$'th
entry is given by $u_{n}^{m}(\ell),\mu_{n}^{m}(\ell)$ and $V_{E},T_{E}$
be the matrices whose $(me,\ell\ell')$'th entries are given by $v_{n}^{m}(\ell,\ell'),\mu_{n}^{m}(\ell,\ell')$.
The objective is quadratic in $(Y_{E}-T_{E})$ and $(Y_{N}-T_{N})$,
so we can flip them to simplify some signs later. Write $W=T-Y$.
Then the objective is
\begin{eqnarray}
\min_{T_{N},T_{E}\in T^{{\cal M}}} & = & \frac{1}{2\lambda}\left\langle W_{N}W_{N}^{\top},UU^{\top}\right\rangle +\frac{1}{2\lambda}\left\langle W_{E}W_{E}^{\top},VV^{\top}\right\rangle -H_{\rho}\left(T_{N},T_{E}\right)\nonumber \\
 &  & \frac{1}{2\lambda}\tr\left(W_{N}^{\top}UU^{\top}W_{N}\right)+\frac{1}{2\lambda}\tr\left(W_{E}^{\top}VV^{\top}W_{E}\right)-H_{\rho}\left(T_{N},T_{E}\right)\nonumber \\
 & = & \frac{1}{2\lambda}\left\Vert U^{\top}W_{N}\right\Vert _{F}^{2}+\frac{1}{2\lambda}\left\Vert V^{\top}W_{E}\right\Vert _{F}^{2}-H_{\rho}\left(T_{N},T_{E}\right)\label{eq:obj2}
\end{eqnarray}
with a matricized form of the entropy as
\begin{eqnarray*}
H_{\rho}(T_{N},T_{E}) & = & -\left\langle 1-R_{N},T_{N}\circ\log T_{N}\right\rangle -\left\langle R_{E},T_{E}\circ\log T_{E}\right\rangle 
\end{eqnarray*}
where $\circ$ denotes elementwise multiplication, and $R_{N},R_{E}$
are matrices of reweighting parameters conforming to $T_{N},T_{E}$.
That is, $\left(R_{N}\right)_{nm,:}=\sum_{n'\in\Ne(n)}\rho_{nn'}$
while $\left(R_{E}\right)_{ne,:}=\rho_{nn'}$. From
this form, the gradients are evidently
\begin{eqnarray*}
\frac{\partial H_{\rho}}{\partial T_{N}} & = & -\left(1-R_{N}\right)\circ\left(1+\log T_{N}\right)\\
\frac{\partial H_{\rho}}{\partial T_{N}} & = & -R_{E}\circ\left(1+\log T_{E}\right)
\end{eqnarray*}
So the gradients of our objective are
\begin{eqnarray}
\frac{\partial h}{\partial T_{N}} & = & \lambda^{-1}UU^{\top}\left(T_{N}-Y_{N}\right)+\left(1-R_{N}\right)\circ\left(1+\log T_{N}\right)\label{eq:gtn}\\
\frac{\partial h}{\partial T_{E}} & = & \lambda^{-1}VV^{\top}\left(T_{E}-Y_{E}\right)+R_{E}\circ\left(1+\log T_{E}\right)\label{eq:gte}
\end{eqnarray}

\subsubsection{Linesearch}

The quadratic terms of $h(\mu+\eta d)$, as a function of $\eta$,
are
\begin{eqnarray*}
 &  & \frac{1}{2\lambda}\left(\left\Vert U^{\top}\left(W_{N}+\eta D_{N}\right)\right\Vert _{F}^{2}+\left\Vert V^{\top}\left(W_{E}+\eta D_{E}\right)\right\Vert _{F}^{2}\right)\\
 & = & \frac{1}{2\lambda}\left(\left\Vert U^{\top}W_{N}+\eta U^{\top}D_{N}\right\Vert _{F}^{2}+\left\Vert V^{\top}W_{E}+\eta V^{\top}D_{E}\right\Vert _{F}^{2}\right)\\
 & = & \frac{1}{2\lambda}\left(\left\Vert U^{\top}W_{N}\right\Vert _{F}^{2}+\left\Vert V^{\top}W_{E}\right\Vert _{F}^{2}+\eta\left(\left\langle U^{\top}W_{N},U^{\top}D_{N}\right\rangle +\left\langle V^{\top}W_{E},V^{\top}D_{E}\right\rangle \right)+\eta^{2}\left(\left\Vert U^{\top}D_{N}\right\Vert _{F}^{2}+\left\Vert V^{\top}D_{E}\right\Vert _{F}^{2}\right)\right)
\end{eqnarray*}
The inner products $\left\langle \cdot,\cdot\right\rangle $
are matrix inner products of $C\times L$ and $D\times L^{2}$ matrices.
For the Bethe/TRW entropy, we will just have to treat it as a black
box.

\subsubsection{Block-Coordinate Updates}

For block-coordinate Frank-Wolfe, we can update the gradient and perform
linesearch without computing the full inner products $U^{\top}W_{N}$
and $V^{\top}W_{E}$. Let $\eta$ denote a step size, $D^{m}=(D_{N}^{m},D_{E}^{m})$
denote the step direction for sample $m$, and $U^{m}$, $V^{m}$
be the submatrices containing only those rows for sample $m$. More
precisely, $D^{[m]}=\left(D_{N}^{[m]},D_{E}^{[m]}\right)$ where $D_{N}^{[m]}$
has the same number of rows as $U$ and $D_{E}^{[m]}$ has the same
number of rows as $V$, but only those rows for sample $m$ are nonzero.
Suppose we move to the point $T+\eta D^{[m]}=\left(T_{N}+\eta D_{N}^{[m]},T_{E}+\eta D_{E}^{[m]}\right)$.
Then our new inner products are
\begin{eqnarray*}
U^{\top}\left(T_{N}+\eta D_{N}^{[m]}-Y_{N}\right) & = & U^{\top}W_{N}+\eta U^{m\top}D^{m}\\
V^{\top}\left(\mu_{E}+\eta d^{[m]}-y_{E}\right) & = & V^{\top}W_{E}+\eta V^{m\top}d^{m}
\end{eqnarray*}
Thus we need only add the update terms $\eta U^{m\top}D_{N}^{m}$
and $\eta V^{m\top}D_{E}^{m}$. The algorithm depends on the value
of $T$ only through these inner products.

\section{Experiments for FW-based Inference}
\label{sec:fw-inf}
In Figure~\ref{fig:pm}, we compare Frank-Wolfe (FW) for bipartite perfect matching to the perturb-and-MAP algorithm of~\citetappendix{KeLi2013}, using code obtained from the authors. In (a) and (b) we plot the $l_{\infty}$ distance of the approximate marginals from exact marginals computed with brute force v.s. the number of calls to the maximum-matching solver. We use a bipartite graph with 10 nodes on each side, i.i.d. edge weights distributed $\text{unif}[0,1]$, and inverse temperatures of 10 (a) and 0.25 (b). In (c),  we run each algorithm for a very large number of MAP calls for a range of temperatures in order to identify the affect of temperature on the algorithms' errors, and results are aggregated over  10 random $n=10$ graphs. Overall, we find that the Bethe approximation provided by FW  is substantially more accurate than perturb-and-MAP and that FW converges more quickly. However, (c) suggests that changes in temperature affect the algorithms' approximation accuracies differently. 

\begin{figure}[h]
       \centering
                \includegraphics[width=0.3\columnwidth]{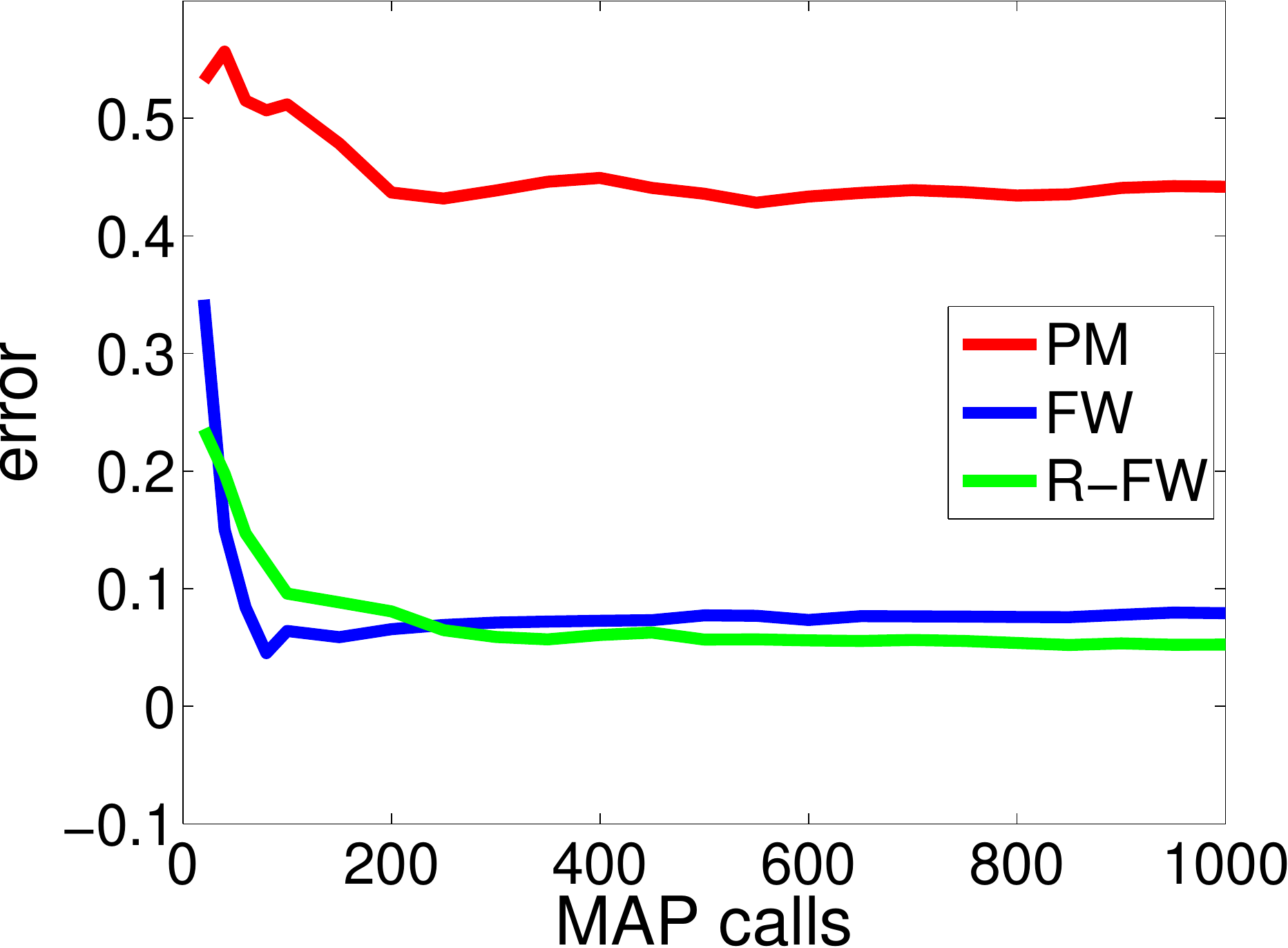}
~
                \includegraphics[width=0.3\columnwidth]{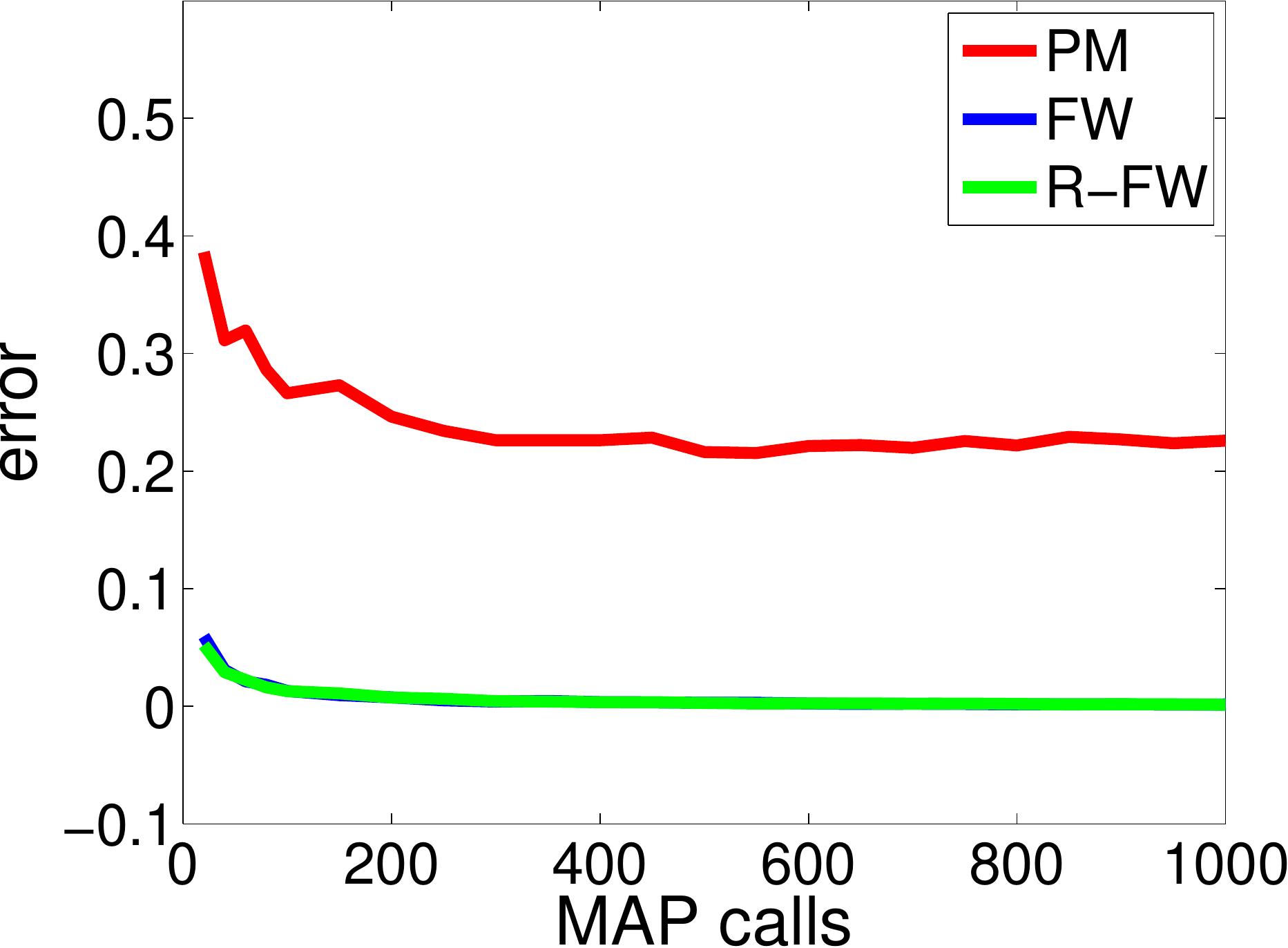}
                \includegraphics[width=0.3\columnwidth]{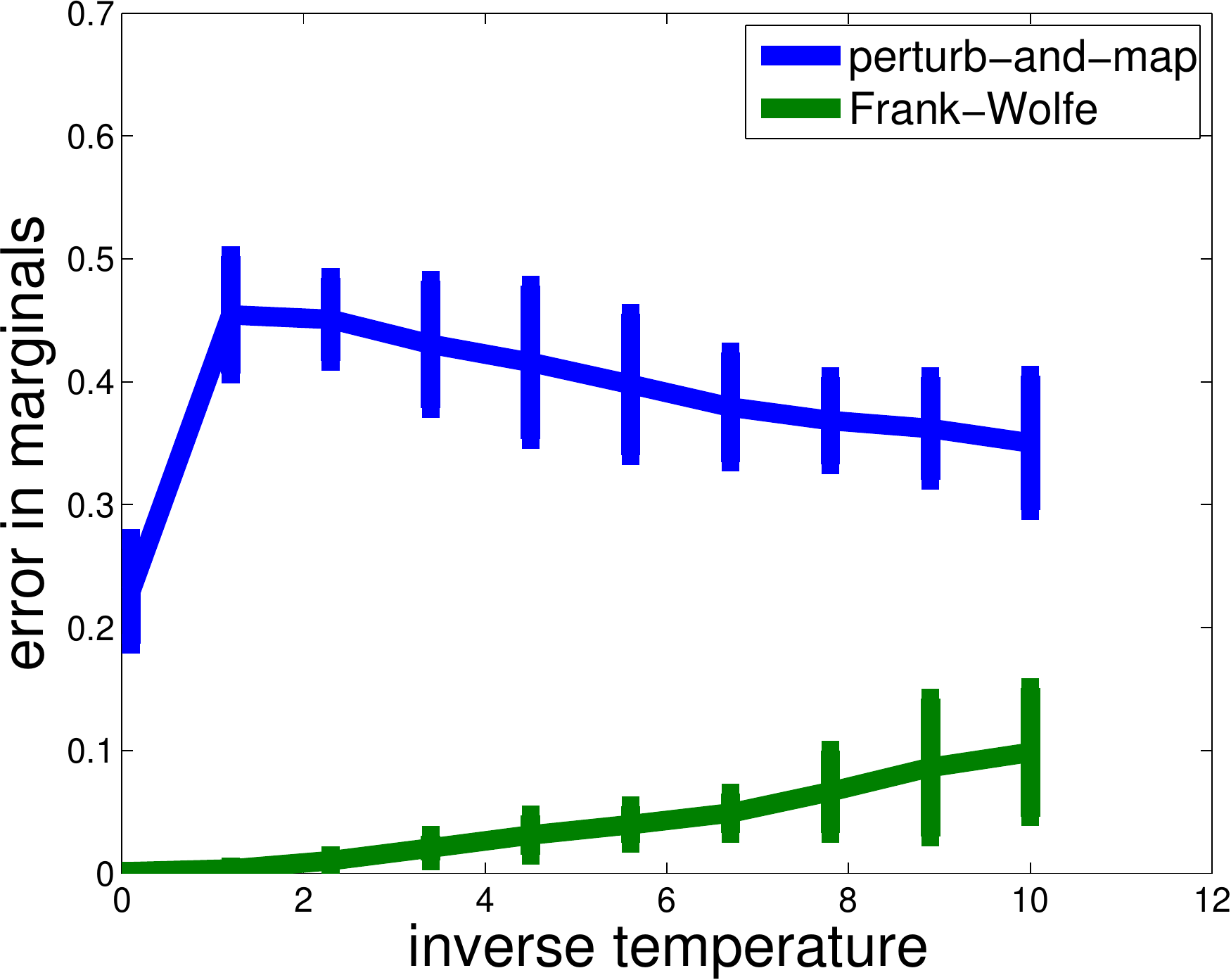}
\caption{Frank-Wolfe v.s. Perturb-and-MAP}\
\label{fig:pm}
\end{figure}

Our proposed FW  algorithm for bipartite perfect matching and the belief propagation algorithm of~\citetappendix{huang2009approximating} minimize the same objective over the same polytope, so we focus on the speed-accuracy trade-offs of the algorithms. In Table~\ref{tab:bpm}, 'rand-$n$` refers to random complete bipartite graphs with n nodes on each side and i.i.d. edge weights from $\text{unif}[0,1]$. The `lda-20' experiment aligns topics from different runs of a Gibbs sampler for an LDA topic model with 20 topics.  All results are averages over 10 graphs. We run the algorithms with a range of termination tolerances in order to obtain various speed-accuracy points. Then, for a range of  $l_{\infty}$ distances to the true Bethe marginals, we compute the time necessary to achieve the specified error. The table presents the ratio of computation time for BP to FW (ratio $>$ 1 means FW is faster). As expected from an algorithm that is no faster than $O(\frac{1}{t})$, we find that it gives good accuracy quickly, but is slow to converge to within very tight error tolerances. Such a speed-accuracy trade off affects other first order methods such as SGD, but can still be advantageous in many cases, including large-scale applications or when optimizing beyond the statistical error of the problem is pointless.

\begin{table}
\centering
\begin{tabular}{ |c | c | c | c |c| }
\hline
$l_\infty$ error& .05  & .01 & 0.005 & 0.001\\
\hline
rand-20 & 3.37 & 1.56 & 1.12 & 0 .25\\
\hline
rand-50 & 18.0 & 11.76 & 5.18  & 0.91 \\
\hline
rand-75 & 23.87  & 23.87 & 9.8  & 1.5\\
\hline
rand-100 & 19.24 & 19.24 & 10.6 & 1.75\\
\hline
lda-20 & 1.6 & .61 & .34 & .10\\
\hline
\end{tabular}
 \caption{
Frank-Wolfe speedup over BP for various error tolerances.}
\label{tab:bpm}
\end{table}

\bibliographystyleappendix{icml2015}
\bibliographyappendix{biblio}

\end{document}